\documentclass[reqno]{amsart}

 \usepackage[foot]{amsaddr}
\usepackage{amsthm}
\usepackage{enumerate}
\usepackage{color}
\usepackage{float}
\usepackage{bbm}
\usepackage{mathrsfs}
\usepackage{amsmath}
\usepackage{comment}
\usepackage{listings}
\usepackage[dvipsnames]{xcolor}
\usepackage{fourier}
\usepackage[ruled,noend]{algorithm2e}
\SetKwInOut{Input}{Input}
\SetKwInOut{Output}{Output}
\SetKw{Return}{Return}
\usepackage{hyperref,cleveref}
\usepackage[colorinlistoftodos]{todonotes}
\hypersetup{
	colorlinks,
	citecolor=blue,
	filecolor=blue,
	linkcolor=blue,
	urlcolor=blue
}
\allowdisplaybreaks

\newtheorem{theorem}{Theorem}[section]
\newtheorem{proposition}[theorem]{Proposition}
\newtheorem{lemma}[theorem]{Lemma}

\newtheorem{corollary}[theorem]{Corollary}

\theoremstyle{definition}
\newtheorem{example}[theorem]{Example}

\newtheorem{remark}[theorem]{Remark}

\newcommand{\R}{\mathbb{R}}

\newcommand{\E}{\mathbb{E}}

\newcommand{\diag}{\textnormal{diag}}

\newcommand{\one}{\mathbbm{1}}

\newcommand{\tr}{\textnormal{tr}}

\newcommand{\PP}{\mathcal{P}}

\newcommand{\Cov}{\operatorname{Cov}}
\newcommand{\Var}{\operatorname{Var}}

\newcommand {\la} {\langle}
\newcommand {\ra} {\rangle}

\newcommand{\bracket}[1]{\left\la#1\right\ra}

\allowdisplaybreaks

\title{On Wasserstein distances for affine transformations of random vectors}

\begin{document}

% \author{Alexander Cloninger$^{1,2}$}\address{$^1$Department of Mathematics, University of California, San Diego, CA}
% \address{$^2$Halicio\u{g}lu Data Science Institute, University of California, San Diego, CA}\email{acloninger@ucsd.edu}

\author{Keaton Hamm$^{1,2}$}
\address{$^1$Department of Mathematics, University of Texas at Arlington, Arlington, TX}
\address{$^2$Division of Data Science, College of Science, University of Texas at Arlington, Arlington, TX
}
\email{keaton.hamm@uta.edu}

\author{Andrzej Korzeniowski$^1$}\email{korzeniowski@uta.edu}

%\keywords{Optimal Transport, Dimensionality Reduction, Wasserstein Space, Multidimensional Scaling, Isomap}
%\subjclass[2020]{49Q22, 60D05, 68T10}
%%%%%%%%%%%%%%%%%%%%%%%%%%%%%%%%%%%%%%%%%%%
%%%%%%%%%%%%%%%%%%%%%%%%%%%%%%%%%%%%%%%%%%%
%%%%%%    Abstract                 %%%%%%%%
%%%%%%%%%%%%%%%%%%%%%%%%%%%%%%%%%%%%%%%%%%%

\maketitle

\begin{abstract}
    We expound on some known lower bounds of the quadratic Wasserstein distance between random vectors in $\R^n$ with an emphasis on affine transformations that have been used in manifold learning of data in Wasserstein space. In particular, we give concrete lower bounds for rotated copies of random vectors in $\R^2$ by computing the Bures metric between the covariance matrices. We also derive upper bounds for compositions of affine maps which yield a fruitful variety of diffeomorphisms applied to an initial data measure. We apply these bounds to various distributions including those lying on a 1-dimensional manifold in $\R^2$ and illustrate the quality of the bounds. Finally, we give a framework for mimicking handwritten digit or alphabet datasets that can be applied in a manifold learning framework.
\end{abstract}

\section{Introduction}

Use of optimal transport as a powerful tool in data analysis and machine learning has increased significantly in recent years. The power of the substantial theoretical underpinnings of optimal transport make it particularly appealing, as does its success in the various ways it has been deployed. Of primary importance is the fact that optimal transport gives an optimal (in terms of a prescribed cost function) way of transforming one data measure into another. Moreover, the associated Wasserstein distance defines a metric on the space of probability measures with an appropriate number of finite moments, which facilitates its use in machine learning or image processing contexts. While Euclidean data is abundantly common, oftentimes one can also view data as probability measures in some natural way (e.g., pixel images can be mapped to a uniform grid and assigned a discrete measure based on normalized intensity). 

One the many ways optimal transport can be used in machine learning is in manifold learning, which is a technique that assumes data lives on a low-dimensional manifold embedded in a high-dimensional space (in this case, the Wasserstein space). Understanding submanifolds of Wasserstein space is the subject of ongoing research \cite{cloninger2023linearized,hamm2023manifold,hamm2022wassmap,liu2022wasserstein,mathews2019molecular,negrini2023applications,wang2010optimal}. Toward this, several works have considered manifolds generated by translating, dilating, or rotating a single probability measure. Most of these techniques then require computing the Wasserstein distances between the original measure and its transformed copies.  

In easy cases like translation and dilation, one can explicitly compute these distances, but for rotations and more sophisticated transformations this is highly non-trivial. There are some known upper and lower bounds for the Wasserstein distance \cite{dowson1982frechet,gelbrich1990formula} and equality of the lower bound is known to hold for Gaussians \cite{givens1984class} and certain structured distributions (e.g., elliptically contoured distributions \cite{gelbrich1990formula}, see Section \ref{SEC:LowerBoundIntro}). Recent manifold learning results \cite{cloninger2023linearized,hamm2022wassmap,negrini2023applications} show that one can isometrically embed translation and dilation manifolds into Euclidean space. Rotational manifolds cannot in general embed isometrically into Euclidean space, partially due to the fact that the Wasserstein distance between rotated measures is not necessarily parametrized by the rotation in an obvious way \cite{negrini2023applications}. Nonetheless, some of the works above provide empirical evidence that such embeddings can still sometimes find the structure of a rotational manifold (e.g., embeddings are very close to a circle in $\R^2$).  These works built on earlier works utilizing dimensionality reduction methods in Wasserstein space, as well as linearized optimal transport (LOT) which approximates Wasserstein distances via a tangent plane approximation; see, for example \cite{kolouri2017optimal,merigot2020quantitative,wang2010optimal,wang2013linear}.

To better understand rotational submanifolds of Wasserstein space, we expound on the lower bounds in terms of the Bures metric. In particular, for $2$-dimensional rotations of random vectors, we give an explicit formula for the lower bound in terms of the means and variances of each component, their covariance, and the rotation angle. We illustrate the sharpness (or not) of the lower bound for rotations in cases when the lower bound is known not to be attained. In many cases, we see that it provides a good estimation of the actual Wasserstein distance.

Additionally, we consider compositions of translations, dilations, and rotations as a way to model more complex transformations of data, and provide an upper bound for some such transformations. Finally, we show how one can generate handwritten alphabet letters on which these transformations can be applied to give rise to a synthetic handwriting dataset that can be used for study of manifold learning algorithms in a controlled setting.

%Many algorithms compute a pairwise distance matrix whose entries are $W_2(\mu_i,\mu_j)$ to approximate the data manifold and compute the embedding. Since the Wasserstein distance is costly to compute ($O(n^3\log n)$ for measures with $n$ Dirac masses), it is therefore undesirable to perform too many of these computations. Consequently, much work has been done to try to approximate Wasserstein distances with lower computational cost. 

\subsection{Notation and definitions}

Let $\mu $ be a probability measure on ${\R^n}$, and let $y\in\R^n$ and $A\in\R^{n\times n}$ be fixed. Let $T:{\R^n} \to {\R^n}$ be an affine map given by 
\[T(x) = y+Ax, \quad x\in\R^n.\]
Then $T$ induces a probability measure as follows:
\[
    \mu_T(B):=\mu(T^{-1}(B)),\quad B\in\mathcal{B}(\R^n),
\]
where $\mathcal{B}(\R^n)$ is the Borel $\sigma$-algebra on $\R^n$. The measure $\mu_T$ is called the \textit{pushforward} of $\mu$ under $T$ and may be denoted $T_\#\mu$ (e.g., \cite{villani2003topics}).

Any probability measure $\mu$ can be viewed as the distribution of a random vector defined on some probability space $(\Omega,\Phi,P)$. Namely, take $(\Omega,\Phi,P)=(\R^n,\mathcal{B}(\R^n),\mu)$ and consider
\[X:(\R^n,\mathcal{B}(\R^n),\mu) \to (\R^n,\mathcal{B}(\R^n),P_X)\]
defined by the identity map $X(x)=x$, where $P_X$ is the distribution of $X$ given by
\[P_X(B) = P\left(\omega\in\Omega \;|\; X(\omega)\in B\right) = \mu\left(x\in\R^n : X(x)\in B\right) = \mu(X^{-1}(B)) = \mu(B).\]
In this case, $\mu$ is often called the \textit{law} of the random vector $X$.

In other words, for any probability measure $\mu$ on $\R^n$, there exists a random vector $X\in \R^n$ whose probability distribution $P_X$ coincides with $\mu$. Consequently, one can describe measure transformation $\mu\mapsto\mu_T$ in terms of the distribution of the corresponding random vectors $X\mapsto T(X)$ by
\[
    P_X\mapsto P_{T(X)},\quad \textnormal{where}\quad P_{T(X)}(B) = P_X(T^{-1}(B)).
\]

Two salient properties of random vectors are their mean and covariance matrices defined by
\[m_X = \mathbb{E}[X],\quad \textnormal{and}\quad \Sigma_X = \mathbb{E}\left[(X-m_X)(X-m_X)^\top\right],\]
respectively. We say that a random vector $X\in\R^n$ with components $X_1,\dots,X_n$ has \textit{uncorrelated components} if $\E[X_iX_j]=0$ for all $i\neq j$. The variance of a component is $\Var[X_i]=\E[X_i^2]-\E[X_i]^2$.

In what follows, we express the quadratic Wasserstein distance between two measures $\mu$ and $\nu$ in terms of their corresponding random vectors $X$ and $Y$ via
\begin{equation}\label{EQN:W2RandomVariables}W_2(\mu,\nu):=W_2(X,Y):=\min\left\{(\E|X-Y|^2)^\frac12:X\sim\mu, Y\sim \nu\right\},\end{equation}
where $X\sim\mu$ means that the law of $X$ is $\mu$, and similarly for $Y\sim\nu$. Some basic properties are that $W_2$ defines a metric on the subset of $\PP(\R^n)$ of measures with finite second moment (which is called the Wasserstein space and denoted $W_2(\R^d)$), and is intimately related to optimal transport theory, which has found application in a vast array of fields including PDEs, signal processing, and machine learning, to name a very few. For seminal works on optimal transport theory, see \cite{villani2003topics,villani2008optimal}, and for more applied and computational treatments \cite{peyre2019computational,santambrogio2015optimal}.
% {\color{red}Maybe better to write it as \[W_2(\mu,\nu):=W_2(X,Y):=\min\left\{(\E|X-Y|^2)^\frac12:X\sim\mu, Y\sim \nu\right\}.\]}

\begin{proposition}\label{PROP:W2MinExp}
Given $\mu,\nu\in\PP(\R^n)$ with corresponding random vectors $X,Y\in\R^n$, we have
\begin{equation}\label{EQN:W2MinExp}
W_2(\mu,\nu)^2 = |m_X-m_Y|^2 + \min_{X\sim\mu,Y\sim\nu}\E|(X-m_X)-(Y-m_Y)|^2.\end{equation}
\end{proposition}
\begin{proof}
    Note that
    \begin{align*}\E|X - Y|^2 & = \E|(X - m_X) - (Y - m_Y) + (m_X - m_Y)|^2\\
    & = \E|m_X-m_Y|^2 + \E|(X-m_X)-(Y-m_Y)|^2 + 2\E\left[\left((X-m_X)-(Y-m_Y)\right)(m_X-m_Y)\right]\\
    & = |m_X-m_Y|^2 + \E|(X-m_X)-(Y-m_Y)|^2 + 2(m_X-m_Y)\left(\E[X-m_X]-\E[Y-m_y]\right)\\
    & = |m_X-m_Y|^2 + \E|(X-m_X)-(Y-m_Y)|^2,
    \end{align*}
    where the final equality comes from the fact that $\E[X-m_X]=\E[Y-m_Y]=0$. The result follows from putting the right hand side above into \eqref{EQN:W2RandomVariables}.
\end{proof}

The utility of the above proposition is that computing the quadratic Wasserstein distance between two random vectors requires only computing the minimizer of the second term in \eqref{EQN:W2MinExp}. 

\subsection{Lower bounds for the Wasserstein distance}\label{SEC:LowerBoundIntro}

Proposition \ref{PROP:W2MinExp} implies that to better understand the Wasserstein distance, one needs to understand $\E|X-Y|^2$ for zero-mean $X$ and $Y$.  To that end, we invoke the following result of Dowson and Landau; note that they call $W_2$ the Fr\'{e}chet distance, which appears to be more widely used in probability theory.

\begin{theorem}[\cite{dowson1982frechet}]\label{THM:FrechetBounds} Let $X,Y$ be random vectors in $\R^n$ having zero mean and covariance matrices $\Sigma _X$ and $\Sigma _Y$, respectively. Then 
\[ \E|X - Y|^2 \geq \tr[\Sigma _X + \Sigma _Y - 2(\Sigma _X \Sigma _Y)^\frac12]\] %\leq \tr[\Sigma _X + \Sigma _Y + 2(\Sigma _X \Sigma _Y)^\frac12]\]
where the square roots are the positive roots. Moreover, the bound is attained when $X-Y$ has covariance matrix 
\[\Sigma_{X - Y} = \Sigma _X + \Sigma _Y - [(\Sigma _X \Sigma _Y)^\frac12 + (\Sigma _Y \Sigma _X)^\frac12].\] If $\Sigma_X$ is non-singular, then this occurs when $X$ and $Y$ are related by \[Y =  \Sigma _X^{-1}(\Sigma _X \Sigma _Y)^\frac12 X.\] 
\end{theorem}

Dowson and Landau proved an upper bound for $\E|X-Y|^2$ in which the $-$ is replaced with $+$ inside the trace formula. This is typically less useful as it is often far from a sharp bound as will be seen in the examples later in the paper. Most of our work is focused on illustrating the quality of the lower bound, hence we focus on in the theorem above. Note also that $B(X,Y):=\tr[\Sigma_X+\Sigma_Y-2(\Sigma_X\Sigma_Y)^\frac12]$ defines a metric (called the Bures metric \cite{bures1969extension,forrester2016relating}) on the space of covariance matrices, which is identical to the space of symmetric positive semidefinite matrices.

Now, combining Proposition \ref{PROP:W2MinExp} and Theorem \ref{THM:FrechetBounds} yields the following. 

\begin{corollary}\label{COR:W2TraceEquality}
Let $\mu,\nu\in\PP(\R^n)$ with corresponding random vectors $X, Y\in\R^n$ having means $m_X, m_Y$ and covariance matrices $\Sigma_X, \Sigma_Y$. Then,
\begin{equation}\label{EQN:W2MeanTraceNorm} W_2(\mu,\nu)^2 = W_2(X,Y)^2 \geq |m_X - m_Y|^2 + \tr[\Sigma _{X - {m_X}} + \Sigma _{Y - {m_Y}} - 2(\Sigma _{X - {m_X}} \Sigma _{Y - {m_Y}})^\frac12].\end{equation} %\\ \leq  |m_X - m_Y|^2 + \tr[\Sigma _{X - {m_X}} + \Sigma _{Y - {m_Y}} + 2(\Sigma _{X - {m_X}} \Sigma _{Y - {m_Y}})^\frac12].\end{multline}
\end{corollary}

While Corollary \ref{COR:W2TraceEquality} follows from the result of Dowson and Landau \cite{dowson1982frechet}, Gelbrich \cite{gelbrich1990formula} proves a similar result and gives more conditions under which the lower bound is attained (see also the book of Peyr\'{e} and Cuturi \cite[Remarks 2.31 and 2.32]{peyre2019computational}). In particular, Gelbrich proves that the lower bound of Corollary \ref{COR:W2TraceEquality} is attained if 
\begin{equation}\label{EQN:Gelbrich}\bar\nu\circ P_\mu^{-1} = \bar\mu\circ (\Sigma_X^{-\frac12}(\Sigma_X^{-\frac12}(\Sigma_X^\frac12\Sigma_Y\Sigma_X^\frac12)^\frac12\Sigma_X^{-\frac12})^{-1},\end{equation} where $\bar\mu$ is the centered version of $\mu$ and $P_\mu$ is the orthogonal projection of $\R^n$ onto the image (column space) of $\Sigma_X$.  This condition can be difficult to check, and thus a further result of Gelbrich in the same paper proves that this condition is satisfied by so-called \textit{elliptically contoured distributions}. An absolutely continuous measure $\mu$ with density $f$ is in this family if, for any $x$, one has 
\[f(x) = \frac{1}{\sqrt{\det(M)}}h\left(\bracket{x-m_X,M^{-1}(x-m_X)}\right),\]
where $M$ is some positive semidefinite matrix and $h$ is a nonnegative function such that $\int_{\R^n}h(\bracket{x,x})dx=1$.  To satisfy the lower bound, $\nu$ must have a density of the form above for the same $h$ but possibly a different positive semidefinite matrix. The name here comes from the fact that Gaussian distributions and uniform measures on ellipsoids satisfy the conditions above.

Prior to Gelbrich's work, it was known that equality holds for the lower bound of Corollary \ref{COR:W2TraceEquality} when $\mu$ and $\nu$ are both Gaussian measures \cite{dowson1982frechet,givens1984class}. The arguments of these papers essentially use an optimization result of Olkin and Pukelsheim \cite{olkin1982distance}.

It should be noted that in most references other than \cite{dowson1982frechet}, the Bures metric on covariance matrices is expressed as $\tr[\Sigma_X+\Sigma_Y-2(\Sigma_X^\frac12 \Sigma_Y\Sigma_X^\frac12)^\frac12]$. We note that these are in fact the same.

\begin{proposition}\label{PROP:TraceNormEquivalence}
Let $A, B \in\R^{n\times n}$ be symmetric positive semi-definite matrices. Then 
\[\tr[(AB)^\frac12] = \tr[(A^\frac12 BA^\frac12)^\frac12].\]
\end{proposition}
\begin{proof}
Note that $A^\frac12 BA^\frac12$ is symmetric, positive semidefinite and has the same eigenvalues as $AB=A^\frac12 A^\frac12 B$ (using the basic fact that $CD$ and $DC$ have the same eigenvalues for any matrices $C$ and $D$). Therefore, the square roots of these matrices have the same eigenvalues.  The conclusion follows from the fact that the trace of a square matrix is the sum of its eigenvalues.
\end{proof}

\section{Examples of Wasserstein distances between transformations of a fixed random vector}

We now undertake the task of explicitly computing the Wasserstein distance between a random vector and various simple transformations of it. The examples here are motivated by simple articulations that might arise in an image analysis setting. Depending on the task one wants to perform, one may want to mod out by the articulation (as in group invariant machine learning) or understand an efficient parametrization of it (as in manifold learning). 

The first two examples, in which we explicitly compute the Wasserstein distance between a random vector and its translates or its (anisotropic) dilations are well-known, but our purpose here is to illustrate how both terms of Proposition \ref{PROP:W2MinExp} contribute to the Wasserstein distance in these cases.

For translations and dilations of a fixed random vector, equality is achieved in \eqref{EQN:W2MeanTraceNorm} as will be seen below. However, computing the Wasserstein distance between $X$ and a rotated copy of $X$ is unexpectedly tricky. Indeed, Brenier's celebrated theorem \cite{brenier1991polar} implies that the rotation matrix cannot be the optimal transport map from $X$ to its rotated copy because it is not the gradient of a convex function.  In our third example, we compute the lower bound of Corollary \ref{COR:W2TraceEquality} explicitly for rotations in $\R^2$ for zero-mean random vectors, and illustrate in subsequent sections how close this bound is in some instances to the actual Wasserstein distance.

\subsection{Translations}

Given $\alpha\in\R^n$, define the translation operator via 
\[T_\alpha(x) = x+\alpha.\]

\begin{proposition}\label{PROP:TranslationW2}
Let $\mu\in\PP(\R^n)$ with corresponding random vector $X\in\R^n$, and let $\Theta\subset\R^n$ be a set of translation vectors with corresponding translation operators $\{T_\theta:\theta\in\Theta\}$. For any $\theta,\theta'\in\Theta$,
\[W_2(T_\theta X,T_{\theta'}X) = |\theta-\theta'|.\] Moreover, the lower bound of \eqref{EQN:W2MeanTraceNorm} is attained.
\end{proposition}
\begin{proof}
    Note that \[m_{T_\theta X} = \E[X+\theta] = \E[X]+\theta = m_X+\theta.\]
    Consequently, $\E|(X+\theta-m_{X+\theta})-(X+\theta'-m_{X+\theta'})|^2 = 0$, so $W_2(T_\theta X,T_{\theta'}X) = |m_{X}+\theta-(m_X+\theta')|=|\theta-\theta'|$ by Proposition \ref{PROP:W2MinExp}.

    To see that the lower bound of \eqref{EQN:W2MeanTraceNorm} is attained, note that
    \[\Sigma_{X+\theta-m_{X+\theta}} = \Sigma_{X+\theta-(m_X+\theta)} = \Sigma_{X-m_X}\]
    for any $\theta,$ hence the trace part of the lower bound vanishes and the conclusion follows.
\end{proof}

Note that one can also verify directly that translations of random vectors satisfy condition \eqref{EQN:Gelbrich}, which guarantees that the lower bound of \eqref{EQN:W2MeanTraceNorm} is attained.  In the case of translations, we see that the difference in means is the only one that contributes to the Wasserstein distance, whereas the trace term vanishes.

\subsection{Dilations}

Given $\lambda\in\R^n$ with no $0$ coordinates, define the scaling operator via
\[S_\lambda(x) = \diag(\lambda_1,\dots,\lambda_n)x.\]

We consider dilations of random vectors in $\R^n$ with uncorrelated components ($\E[X_iX_j]=0$ for all $i\neq j$).  The following proposition computes the Wasserstein distance between two scalings of a given random vector of this form. 

\begin{proposition}\label{PROP:DilationW2}
Suppose that $\mu\in\PP(\R^n)$ has associated random vector $X = (X_1,\dots,X_n)\in\R^n$ with uncorrelated components, and hence diagonal covariance matrix given by $\Sigma_{X-m_X} = \diag(\Var[X_1],\dots,\Var[X_n]).$ Let $\Lambda\subset\R^n$ be a set of scaling parameters such that $\lambda\in\Lambda$ satisfies $\lambda_i\neq0$ for all $i$, and let the associated scaling operators be $\{S_\lambda:\lambda\in\Lambda\}$. Then for any $\lambda,\lambda'\in\Lambda$,
\[W_2(S_\lambda X,S_{\lambda'}X)^2 = \sum_{i=1}^n \E[X_i^2]|\lambda_i-\lambda'_i|^2.\]
In particular, the lower bound of \eqref{EQN:W2MeanTraceNorm} is attained.
\end{proposition}

\begin{proof}

Since $m_{S_\lambda X}=\E[S_\lambda X] = (\lambda_1\E[X_1],\dots,\lambda_n\E[X_n])$, we have
\[m_{S_\lambda X}-m_{S_{\lambda'}X} = ((\lambda_1-\lambda'_1)\E[X_1],\dots,(\lambda_n-\lambda'_n)\E[X_n]).\]
Thus,
\[|m_{S_\lambda X}-m_{S_{\lambda'}X}|^2 = \sum_{i=1}^n\E[X_i]^2|\lambda_i-\lambda'_i|^2.\]

Now we compute $\E[|(S_\lambda X - m_{S_\lambda X})-(S_{\lambda'}X-m_{S_{\lambda'}X})|^2]$ as follows.
\begin{align*}
\E[|(S_\lambda X - m_{S_\lambda X})-(S_{\lambda'}X-m_{S_{\lambda'}X})|^2] & = \E[|(\lambda_i-\lambda_i')X_i)_i - ((\lambda_i-\lambda_i')\E[X_i])_i|^2]\\
& = \E[|((\lambda_i-\lambda_i')(X_i-\E[X_i]))_i|^2]\\
& = \sum_{i=1}^n|\lambda_i-\lambda_i'|^2\E[|X_i-\E[X_i]|^2]\\
& = \sum_{i=1}^n|\lambda_i-\lambda_i'|^2(\E[X_i^2]-\E[X_i]^2).
\end{align*}
Consequently, \[W_2(S_\lambda X,S_{\lambda'}X)^2 = \sum_{i=1}^n (\E[X_i]^2+\E[X_i^2]-\E[X_i]^2)|\lambda_i-\lambda_i'|^2 = \sum_{i=1}^n\E[X_i^2]|\lambda_i-\lambda_i'|^2\]
as claimed.

Finally, to see that this is attained via the lower bound of \eqref{EQN:W2MeanTraceNorm}, note that one has $\Sigma_{S_\lambda X} = \diag(\lambda_1^2\Var[X_1],\dots,\lambda_n^2\Var[X_n]).$ Hence,
\begin{align*}\tr\left(\Sigma_{S_\lambda X} + \Sigma_{S_{\lambda'}X} - 2\left(\Sigma_{S_\lambda X}\Sigma_{S_{\lambda'}X}\right)^\frac12\right)
& = \sum_{i=1}^n \left(\lambda_i^2+\lambda_i'^2-2\lambda_i\lambda'_i\right)\Var[X_i]\\
& = \sum_{i=1}^n |\lambda_i-\lambda'_i|^2\Var[X_i]\\
& = \sum_{i=1}^n|\lambda_i-\lambda_i'|^2(\E[X_i^2]-\E[X_i]^2).
\end{align*}
\end{proof}

In contrast to the translation case, the trace portion of \eqref{EQN:W2MeanTraceNorm} plays a pivotal role in the Wasserstein distance for dilations, and actually cancels out the term involving the difference in means.

\begin{remark}\label{REM:Dilation}
    Proposition \ref{PROP:DilationW2} should be compared with \cite[Lemma 3.7]{hamm2022wassmap}. The result presented here is more general in one sense: it allows for negative scaling in any direction, which is not allowed in \cite{hamm2022wassmap}. However, Proposition \ref{PROP:DilationW2} is more limited in another sense: it requires the assumption that $X$ has uncorrelated components, which is not required in the result of \cite{hamm2022wassmap}. It is not obvious to the authors if the proof technique used here extends easily to arbitrary random vectors given the difficulty of computing $\tr[(\Sigma_{S_\lambda X}\Sigma_{S_{\lambda'}X})^\frac12].$

    Note that \cite[Lemma 3.7]{hamm2022wassmap} implies that if the scaling vector $\Lambda$ has strictly positive entries and $X$ is an arbitrary random vector (with possibly correlated components), then the conclusion of Proposition \ref{PROP:DilationW2} holds, i.e., $W_2(S_\lambda X,S_{\lambda'}X)^2 = \sum_{i=1}^n \E[X_i^2]|\lambda_i-\lambda_i'|^2$.
\end{remark}

\subsection{Rotations}

In this section, we consider the Wasserstein distance between probability measures under orthogonal transformations on $\R^n$, and give an explicit formula for the lower bound of \eqref{EQN:W2MeanTraceNorm} for rotations in $\R^2$.  To wit, consider a parametrized family of orthogonal transformations $\{Q_\theta:\theta\in\Theta\}\subset O(n)$.  For a random vector $X\in\R^n$ with mean $m_X$ and covariance matrix ${\Sigma_X}$ we have \[{m_{{Q_\theta }X}} = \E[Q_\theta X] = Q_\theta \E[X] = Q_\theta m_X,\] whence, by cyclic properties of trace, \[\tr(\Sigma_{{Q_\theta}X}) = \tr({Q_\theta }{\Sigma_X}Q_\theta ^\top) = \tr(Q_\theta ^\top{Q_\theta }{\Sigma _X}) = \tr({\Sigma _X}).\] Consequently, by Corollary \ref{COR:W2TraceEquality} and the above, we have, for any parameters $\theta,\varphi\in\Theta$, 
\begin{multline}\label{EQN:Orthogonal}
W_2(Q_\theta X,Q_{\varphi}X)^2 \geq |Q_\theta m_{X} - Q_\varphi m_{X}|^2 +\\ \tr[\Sigma_{X - {m_X}} + \Sigma_{X - m_X} - 2(\Sigma_{{Q_\theta X} - m_{Q_\theta X}}\Sigma_{{Q_\varphi X} - m_{Q_{\varphi X}}})^\frac12].\end{multline}  

\begin{remark}\label{REM:OrthogonalTrace} Since $Y$ and $Y+b$ have the same covariance matrix for any random vector $Y$ and constant vector $b$, the trace part in \eqref{EQN:Orthogonal} (the Bures metric) is \[\tr[2{\Sigma_X} - 2(\Sigma _{{Q_\theta X}}\Sigma _{Q_\varphi X})^\frac12].\]
The main challenge in computing this trace is computing the positive square root of
\[\Sigma _{{Q_\theta X}}{\Sigma _{Q_\varphi X}}.\]
\end{remark}

In what follows, we compute a closed form for the lower bound \eqref{EQN:Orthogonal} for rotations in ${\R^2}$. To compute the Bures metric, by Remark \ref{REM:OrthogonalTrace}, it suffices to consider a base probability measure corresponding to a zero-mean random vector $X\in\R^2$. That is, we assume 
\[\E[X] = \E \begin{bmatrix} X_1 \\ X_2\end{bmatrix} = \begin{bmatrix} \E[X_1] \\ \E[X_2]\end{bmatrix} = \begin{bmatrix} 0\\0\end{bmatrix}.\]
% and
% \[\Cov(X_1,X_2) = \E[(X_1 - \E[X_1])\;(X_2 - \E[X_2])] = \E[X_1X_2] = 0.\] 

The family of rotations on $\R^2$ are 
\[ R_\theta = \begin{bmatrix} \cos\theta & -\sin\theta \\ \sin\theta & \cos\theta \end{bmatrix},\qquad \theta\in[0,2\pi),\]
which induce transformations of a random vector $X$ given by
\[R_\theta X = \begin{bmatrix} X_1\cos\theta -X_2\sin\theta \\ X_1\sin\theta + X_2\cos\theta\end{bmatrix}.\]

% {R_\theta }x = \left[ {\begin{array}{*{20}{c}} 

%   {\cos \theta }&{ - \sin \theta } \\  

%   {\sin \theta }&{\cos \theta }  

% \end{array}} \right]\left[ {\begin{array}{*{20}{c}} 

%   {{x_1}} \\  

%   {{x_2}}  

% \end{array}} \right]{\text{ ,  }}x = \left[ {\begin{array}{*{20}{c}} 

%   {{x_1}} \\  

%   {{x_2}}  

% \end{array}} \right]{\text{ , }}\theta  \in {\text{[0, 2}}\pi {\text{]}}\]
% and induce transformations of the random vector \[X = \left[ {\begin{array}{*{20}{c}} 

%   {{X_1}} \\  

%   {{X_2}}  

% \end{array}} \right]\]
% given by ${R_\theta }X = \left[ {\begin{array}{*{20}{c}} 

%   {{X_1}\cos \theta }&{ - {X_2}\sin \theta } \\  

%   {{X_1}\sin \theta }&{{X_2}\cos \theta }  

% \end{array}} \right].$

To derive an explicit formula for the Bures metric in \eqref{EQN:Orthogonal} between $R_\theta X$ and $R_{\varphi}X$, we need several auxiliary facts which we collect below. %To derive L2- Wasserstein distance formula 
% ${W_2}({R_\theta }X,{R_{\theta '}}X)$
% we need several auxiliary facts that will be shown in the following propositions. 

\begin{proposition}\label{PROP:RotationTrace} Let $X = \begin{bmatrix}X_1\\ X_2\end{bmatrix}\in\R^2$ be a zero-mean random vector. Then the trace of the covariance matrix of the rotation of $X$ is $\theta$-invariant, and we have for any $\theta\in[0,2\pi)$, \begin{equation}\label{EQN:RotationTrace}\tr[{\Sigma _{{R_\theta }X}}] = \E[X_1^2] + \E[X_2^2].\end{equation}   
\end{proposition}

\begin{proof}
Recall that
\begin{align*}{\Sigma _{{R_\theta }X}} & = \E
\begin{bmatrix} X_1\cos\theta -X_2\sin\theta \\ X_1\sin\theta + X_2\cos\theta\end{bmatrix} 
\begin{bmatrix} X_1\cos\theta  -X_2\sin\theta \\ X_1\sin\theta + X_2\cos\theta\end{bmatrix}^\top, %\\
%  & = \begin{bmatrix} \E[(X_1\cos\theta - X_2\sin\theta)^2] & \ast \\
% \ast & \E[(X_1\sin\theta + X_2\cos\theta)^2]\end{bmatrix}, \\
\end{align*}
which has entries
\begin{align*}
 (\Sigma_{R_\theta X})_{11} & = \E[X_1^2]\cos^2\theta + \E[X_2^2]\sin^2\theta - 2\E[X_1X_2]\cos\theta\sin\theta\\
 (\Sigma_{R_\theta X})_{12} & = (\E[X_1^2]-\E[X_2^2])\cos\theta\sin\theta + \E[X_1X_2](\cos^2\theta-\sin^2\theta) \\
 (\Sigma_{R_\theta X})_{22} & = \E[X_1^2]\sin^2\theta + \E[X_2^2]\cos^2\theta + 2\E[X_1X_2]\cos\theta\sin\theta,
\end{align*}
and $(\Sigma_{R_\theta X})_{21}=(\Sigma_{R_\theta X})_{12}$.
The desired formula for the trace can be seen directly from above.

\end{proof}

\begin{lemma}[\cite{levinger1980square}]\label{LEM:SquareRoot}
Let $M\in\R^{2\times 2}$ be given by    
\[M = \begin{bmatrix} A & B\\ C & D\end{bmatrix}.\]
Then 
\[M^\frac12 = \frac{1}{\sqrt{\tr(M)+2\sqrt{\det(M)}}}\begin{bmatrix} A+\sqrt{\det(M)} & B\\ C & D+\sqrt{\det(M)}\end{bmatrix},\]
and $\tr[M^\frac12] = (\tr(M)+2\sqrt{\det(M)})^\frac12.$
\end{lemma}
 
\begin{proposition}\label{PROP:RotationTraceSquareRoot}
Let $X = \begin{bmatrix}X_1\\X_2\end{bmatrix}$ be a zero-mean random vector, and let $\theta,\varphi\in[0,2\pi)$. Let $a = \E[X_1^2], b = \E[X_2^2],$ and $c=\E[X_1X_2]=\Cov(X_1,X_2).$ Then
% \[
% \tr[(\Sigma _{R_\theta X}\Sigma _{R_\varphi X})^\frac12] = \sqrt{(\E[X_1^2] - \E[X_2^2])^2\cos^2(\theta  - \varphi ) + 4\E[X_1^2]\E[X_2^2]}.
% \]
\[
\tr[(\Sigma _{R_\theta X}\Sigma _{R_\varphi X})^\frac12] = \sqrt{((a-b)^2+4c^2)\cos^2(\theta  - \varphi ) + 4(ab-c^2)}.
\]
\end{proposition}
 
\begin{proof}
Utilizing Proposition \ref{PROP:RotationTrace} and some tedious calculations involving trigonometric reduction formulas, one finds that
\[\tr[\Sigma _{R_\theta X}\Sigma _{R_\varphi X}] = \frac12\left((a+b)^2+((a-b)^2+4c^2)\cos(2(\theta-\varphi))\right)\]
and
\[\det(\Sigma _{R_\theta X}\Sigma _{R_\varphi X}) = \det(\Sigma_X)^2 = (ab-c^2)^2.\]
Note that by Cauchy-Schwarz, $c^2\leq ab$, so $\sqrt{\det(\Sigma _{R_\theta X}\Sigma _{R_\varphi X})} = |ab-c^2| = ab-c^2.$ Applying Lemma \ref{LEM:SquareRoot}, and substituting $2\cos^2(\theta-\varphi)-1=\cos(2(\theta-\varphi))$ yields the desired result.

\end{proof}

\begin{remark}
One may note that in a simpler case in which the components of $X$ are uncorrelated, i.e., $\E[X_1X_2]=0$, the result of Proposition \ref{PROP:RotationTraceSquareRoot} becomes somewhat simpler, and we find that 
\[\tr[(\Sigma _{R_\theta X}\Sigma _{R_\varphi X})^\frac12] = \sqrt{(a-b)^2\cos^2(\theta-\varphi)+4ab}.\]
\end{remark}
 
Combining the two propositions above allows us to compute the trace term of \eqref{EQN:Orthogonal} for any $X$ with uncorrelated components. For $X$ with non-zero mean, we have the following. 
\begin{proposition} \label{PROP:RotationMean}
% Let 
% \[{\E}[X] = \E\begin{bmatrix} X_1 \\ X_2\end{bmatrix} = \begin{bmatrix} \E[X_1] \\ \E[X_2]\end{bmatrix}.\]
Let $X = \begin{bmatrix}X_1\\ X_2\end{bmatrix}\in\R^2$, with $\E[X_1], \E[X_2]$ not necessarily 0.  Then, for any $\theta,\varphi\in[0,2\pi)$,
\[|\E[R_\theta X] - \E[R_\varphi X]|^2 = 2\left(\E[X_1]^2 + \E[X_2]^2\right)(1 - \cos (\theta  - \varphi)).\]
\end{proposition}

 \begin{proof}
\[R_\theta X - R_\varphi X = \begin{bmatrix} X_1(\cos\theta-\cos\varphi)-X_2(\sin\theta-\sin\varphi)\\
X_1(\sin\theta-\sin\varphi) + X_2(\cos\theta-\cos\varphi)\end{bmatrix},\]
whence by taking expected value and then the square of the Euclidean norm we get 
\[ \left(\E[X_1]^2 + \E[X_2]^2\right)\left(2 - 2(\cos \theta \cos \varphi  + \sin \theta \sin \varphi) \right) = 
2\left(\E[X_1]^2 + \E[X_2]^2\right)(1 - \cos (\theta  - \varphi))\]
as claimed. 
\end{proof}

 Combining Corollary \ref{COR:W2TraceEquality} with Remark \ref{REM:OrthogonalTrace} and Propositions \ref{PROP:RotationTrace}, \ref{PROP:RotationTraceSquareRoot}, and \ref{PROP:RotationMean} with $X_i$ replaced by $X_i-\E[X_i]$, we arrive at the following.

% {\color{red}To Do: Fix these references} By (1) and Propositions (1.1) – (1.3) for 
% ${X_i} \to {X_i} - {\E}[{X_i}]$, and the fact that ${\text{Var[}}{X_1}] = {\E}{({X_1} - {\E}[{X_1}])^2} = a$,
% ${\text{Var[}}{X_2}] = {\E}{({X_2} - {\E}[{X_2}])^2} = b$ we arrive at the following.

% \begin{theorem}\label{THM:RotationW2}
% Let $X = \begin{bmatrix}X_1\\ X_2\end{bmatrix}\in\R^2$ have uncorrelated components. Then for any $\theta,\varphi\in[0,2\pi)$,
% \begin{multline}\label{EQN:RotationLowerBound}W_2(R_\theta X,R_\varphi X)^2 \geq 2\left((\E[X_1])^2 +
% (\E[X_2])^2 \right)(1 - \cos (\theta  - \varphi )) +  2(\Var[X_1] + \Var[X_2]) \\ - 2\sqrt{(\Var[X_1] - \Var[X_2])^2 \cos^2(\theta  - \varphi ) + 4\Var[X_1]\Var[X_2]}. \end{multline}
% \end{theorem}

\begin{theorem}\label{THM:RotationW2}
    Let $X= \begin{bmatrix}X_1\\ X_2\end{bmatrix}\in\R^2$ be a random vector. Then for any $\theta,\varphi\in[0,2\pi)$
\begin{multline}\label{EQN:RotationLowerBound}W_2(R_\theta X,R_\varphi X)^2 \geq 2\left((\E[X_1])^2 +
(\E[X_2])^2 \right)(1 - \cos (\theta  - \varphi )) +  2(\Var[X_1] + \Var[X_2]) \\ - 
2\sqrt{((\Var[X_1] - \Var[X_2])^2 +4(\E[X_1X_2])^2)\cos^2(\theta  - \varphi ) + 4(\Var[X_1]\Var[X_2]-(\E[X_1X_2])^2)},\end{multline} 
\end{theorem}

Note that if $R_\theta X$ and $R_\varphi X$ are Gaussians or elliptically contoured distributions, then the lower bound is attained in \eqref{EQN:RotationLowerBound} by the result of Gelbrich.

\subsection{Illustrating the Lower bound for Rotations}

Here we give some further perspectives on the lower bound of Theorem \ref{THM:RotationW2} from various directions.  We begin by noticing that the lower bound is equivalent to the Euclidean distance between the angles (for small difference in angles). In particular, for random vectors for which equality holds in Theorem \ref{THM:RotationW2}, we see that $W_2(R_\theta X,R_\varphi X)\asymp |\theta-\varphi|$ (that is, there are constants $c_1,c_2\geq0$ such that $c_1|\theta-\varphi|\leq W_2(R_\theta X, R_\varphi X)\leq c_2|\theta-\varphi|$) for $|\theta-\varphi|\leq\frac{\pi}{2}$.

\begin{theorem}\label{THM:fbound}
% The Wasserstein distance between probability measures under rotations in $\R^2$, parametrized by the angle of rotation, is equivalent to Euclidean metric in $\R^1$ with the following lower and upper bounds

% \[\left(\frac{4(m_1^2+m_2^2)}{\pi^2}+\frac{4(\sigma_1^2-\sigma_2^2)^2}{(4+\pi^2)(\sigma_1^2+\sigma_2^2)}\right)|\theta  - \varphi|^2 \leq 
% W_2(R_\theta X,R_\varphi X)^2 \leq \left( m_1^2 + m_2^2 + \frac{(\sigma _1^2 - \sigma _2^2)^2}{\sigma _1\sigma _2} \right)|\theta  - \varphi |^2\]
%    for  
% $|\theta  - \varphi| \leq \frac{\pi}{2}$, where  
% ${m_i} = \E[X_i], \sigma _i^2 = \Var[X_i]=\E(X_i - \E[X_i])^2$.
Let $X=\begin{bmatrix} X_1\\X_2\end{bmatrix}$ be a random vector on $\R^2$ with $m_i=\E[X_i],$ $i=1,2$, $\Var[X_1]=a, \Var[X_2]=b$, and $\E[X_1X_2]=c$. Let $g(X,\theta,\varphi)$ be the lower bound of \eqref{EQN:RotationLowerBound}, i.e., 
\[g(X,\theta,\varphi) = 2(m_1^2+m_2^2)(1-\cos(\theta-\varphi))+2(a+b)-2\sqrt{((a-b)^2+4c^2)\cos^2(\theta-\varphi)+4(ab-c^2)}.\]  Then for $|\theta-\varphi|\leq\frac{\pi}{2},$ we have
\[\left(\frac{8(m_1^2+m_2^2)}{\pi^2}+\frac{4((a-b)^2+4c^2)}{(4+\pi^2)(a+b)}\right)|\theta  - \varphi|^2 \leq g(X,\theta,\varphi)
 \leq \left( m_1^2 + m_2^2 + \frac{(a-b)^2+4c^2}{\sqrt{ab-c^2}} \right)|\theta  - \varphi |^2.\]
% \[g(X,\theta,\varphi) = 2(m_1^2+m_2^2)(1-\cos(\theta-\varphi))+2(\sigma_1^2+\sigma_2^2)-2\sqrt{(\sigma_1^2-\sigma_2^2)^2\cos^2(\theta-\varphi)+4\sigma_1^2\sigma_2^2}.\]  Then for $|\theta-\varphi|\leq\frac{\pi}{2},$ we have
% \[\left(\frac{8(m_1^2+m_2^2)}{\pi^2}+\frac{4(\sigma_1^2-\sigma_2^2)^2}{(4+\pi^2)(\sigma_1^2+\sigma_2^2)}\right)|\theta  - \varphi|^2 \leq g(X,\theta,\varphi)
%  \leq \left( m_1^2 + m_2^2 + \frac{(\sigma _1^2 - \sigma _2^2)^2}{\sigma _1\sigma _2} \right)|\theta  - \varphi |^2.\]
\end{theorem}

\begin{proof}
Regarding the upper bound, since $1-\cos x = 2\sin^2\frac{x}{2}$ and $|\sin z| \leq |z|$ we have 
\[2(m_1^2 + m_2^2)(1 - \cos (\theta  - \varphi )) \leq (m_1^2 + m_2^2)|\theta  - \varphi|^2.\]
Next, to find the upper bound on the remaining two terms, we define the following function:
\[f(x) = 2(a + b) - 2\sqrt {((a - b)^2+4c^2)\cos^2x + 4(ab-c^2)}. \] 
Notice that $f(0) = 0$, whereby for every $0<x\leq\frac\pi2$, there exists a $0<\xi<x$ such that $f(x) = f(0) + f'(\xi)x$ by the Mean Value Theorem. Now, 
\[f'(\xi) = \frac{((a - b)^2+4c^2) 2\sin \xi\cos \xi}{\sqrt {((a - b)^2+4c^2)\cos^2\xi + 4(ab-c^2)}}  = \frac{((a - b)^2+4c^2)\sin (2\xi)}{\sqrt{((a - b)^2+4c^2)\cos^2\xi + 4(ab-c^2)}},\]
whence, using that $|\sin(t)|\leq |t|$ and $|\xi|<|x|$,
\[|f(x)|=|f'(\xi)x| \leq \frac{((a - b)^2+4c^2)|\sin (2\xi)||x|}{2\sqrt{ab-c^2} } \leq \frac{((a - b)^2+4c^2)|x|^2}{\sqrt{ab-c^2}}.\] Combining the above estimates with $x=\theta-\varphi$ yields the desired upper bound. 

As before, the first term in the lower bound is
$(m_1^2 + m_2^2)\;4{\sin ^2}\left( {\frac{{\theta  - \varphi }}{2}} \right).$ For
$0 \leq z \leq \frac{\pi}{4}$, we have $\sin z \geq \frac{2\sqrt{2}}{\pi}z$, which gives 
$4{\sin ^2}\left(\frac{\theta  - \varphi }{2} \right) \geq \frac{8|\theta  - \varphi|^2}{\pi ^2}$ plugging in $z = \frac{\theta-\varphi}{2}$. For the rest of the lower bound, we use $\cos^2 x\leq 1-\frac{x^2}{1+x^2}$ for $0\leq x\leq\frac{\pi}{2}$ to see that
\begin{align*}
   & 2(a+b)-2\sqrt{((a-b)^2+4c^2)\cos^2x+4(ab-c^2)} \\ & \geq 2(a+b)-2\sqrt{(a+b)^2-((a-b)^2+4c^2)\frac{x^2}{1+x^2}}\\
   & \geq 2(a+b)\left(1-\sqrt{1-\frac{((a-b)^2+4c^2)}{(a+b)^2}\frac{x^2}{1+x^2}}\right).
\end{align*}
Next, notice that \[\frac{(a-b)^2+4c^2}{(a+b)^2} = \frac{(a+b)^2+4(c^2-ab)}{(a+b)^2}\leq 1\]
since $c^2-ab\leq 0$, and likewise $\frac{x^2}{1+x^2}\leq 1$. We then use the fact that $\sqrt{1-u^2}\leq1-\frac12u^2$ for $0\leq u\leq 1$ to see that the above quantity is at least
\begin{align*}
    2(a+b)\left(1-\left(1-\frac12\frac{((a-b)^2+4c^2)}{(a+b)^2}\frac{x^2}{1+x^2}\right)\right) & \geq \frac{((a-b)^2+4c^2)}{(a+b)}\frac{x^2}{1+x^2}\\
    & \geq \frac{((a-b)^2+4c^2)}{(a+b)}\frac{x^2}{1+\frac{\pi^2}{4}} \\
    & \geq \frac{((a-b)^2+4c^2)}{(a+b)}\frac{4}{4+\pi^2}x^2.
\end{align*}

% is (with $a=\sigma_1^2=\Var[X_1]$ and $b=\sigma_2^2=\Var[X_2]$ and $x=\theta-\varphi$)
% \[2(a+b)-2\sqrt{(a-b)^2\cos^2 x+4ab} = 2(a+b)-2\sqrt{(a+b)^2-(a-b)^2\sin^2x}.\]
% Without loss of generality, we may assume that $a\geq b$, otherwise we can change coordinates of $X$ which will not affect the Wasserstein distance of its rotated copies, hence its lower bound. Then let $A=a+b$ and $B=a-b$, and we have $A>B\geq 0.$  Since $\frac{x}{\sqrt{1+x^2}}\leq\sin x$ for $0\leq x\leq\frac{\pi}{2}$, and $\sqrt{1-u^2}\leq1-\frac12u^2$ for $0\leq u\leq 1,$
% \begin{align*}
%     2A-2\sqrt{A^2-B^2\sin^2x} & \geq 2A-2\sqrt{A^2-B^2\frac{x^2}{1+x^2}}\\
%     & = 2A\left(1-\sqrt{1-\left(\frac{B}{A}\frac{x}{\sqrt{1+x}^2}\right)^2}\right)\\
%     & \geq 2A\left(1-\left(1-\frac12\left(\frac{B}{A}\frac{x}{\sqrt{1+x^2}}\right)^2\right)\right)\\
%     & = \frac{B^2}{A}\frac{x^2}{1+x^2}\\
%     & \geq \frac{B^2}{A}\frac{x^2}{1+\frac{\pi^2}{4}}\\
%     & = \frac{B^2}{A}\frac{4}{4+\pi^2}x^2.
% \end{align*}
Putting these together, we have
\[W_2(R_\theta X,R_\varphi X)^2 \geq \left(\frac{8(m_1^2+m_2^2)}{\pi^2}+\frac{4((a-b)^2+4c^2)}{(4+\pi^2)(a+b)}\right)|\theta-\varphi|^2.\]
\end{proof}

\begin{corollary}
When $X$ is a Gaussian, or any random vector that satisfies equality in the lower bound of Theorem \ref{THM:FrechetBounds} (hence of Theorem \ref{THM:RotationW2}), one can replace $g(X,\theta,\varphi)$ with $W_2(R_\theta X,R_\varphi X)$ in Theorem \ref{THM:fbound}. In particular, for such random vectors, the Wasserstein distance between rotated copies is equivalent to the absolute value of the difference of the rotation angles when $|\theta-\varphi|\leq\frac{\pi}{2}$.
\end{corollary}

For measures which do not satisfy equality of the lower bound of Theorem \ref{THM:RotationW2}, it is unclear how tight the bound is. We now present some examples in which equality may not be achieved, and illustrate the quality of the lower bound.
 
\begin{example}\label{EX:01Rotation}  Let  $X$  be uniformly distributed over the unit square $[0,1]\times[0,1]\subset\R^2.$ Then $X = \begin{bmatrix} X_1 \\ X_2\end{bmatrix}$
where $X_i\sim\mathcal{U}[0,1]$ i.i.d., and are thus uncorrelated. We have densities $\one_{[0,1]}(x_1)$ and $\one_{[0,1]}(x_2)$ with $\E[X_1] = \E[X_2] = \frac{1}{2}, \Var[X_1] = \Var[X_2] = \frac{1}{{12}}$.
Then by Theorem \ref{THM:RotationW2}, for $0\leq|\theta-\varphi|\leq\frac{\pi}{2},$ we have

\begin{multline*}W_2(R_\theta X,R_\varphi X)^2 \geq 2\left(\E[X_1]^2 + \E[X_2]^2 \right)(1 - \cos (\theta  - \varphi )) +2(\Var[X_1]+\Var[X_2])\\ -4\sqrt{\Var[X_1]\Var[X_2]} = 1 - \cos (\theta  - \varphi ).\end{multline*} 

To illustrate the quality of the lower bound in this instance, Figure \ref{FIG:01Rotation} shows both the actual Wasserstein distance of $X$ and $R_\theta X$, the lower bound, and the difference with the lower bound for $\theta\in[0,\frac{\pi}{2}]$. We know that the square does not give equality in the lower bound of \ref{THM:RotationW2}, but nonetheless it appears to be a close approximation in this instance. In this, and subsequent figures, we generate a discrete approximation to the random vectors in question and compute Wasserstein distances via the linear program in the Python Optimal Transport (POT) package \cite{flamary2021pot}. More specifically, if $X_1 \sim \mathscr{U}[0,1]$, we take $t_1,\dots,t_N$ equi-distant points in $[0,1]$ (with $N$ large) and set $\hat X_1 = \frac{1}{N}\sum\delta_{t_i}.$ This is more structured than an empirical measure which could be used as an alternative, but is still a good estimation (e.g., \cite{canas2012learning,hamm2023lattice}). When generating the subsequent figures, we do not see an appreciable difference for values of $N$ larger than 100, so we typically utilize this sample size for computational speed.

\end{example}
\begin{figure}[!h]
\centering
\includegraphics[width=.85\textwidth]{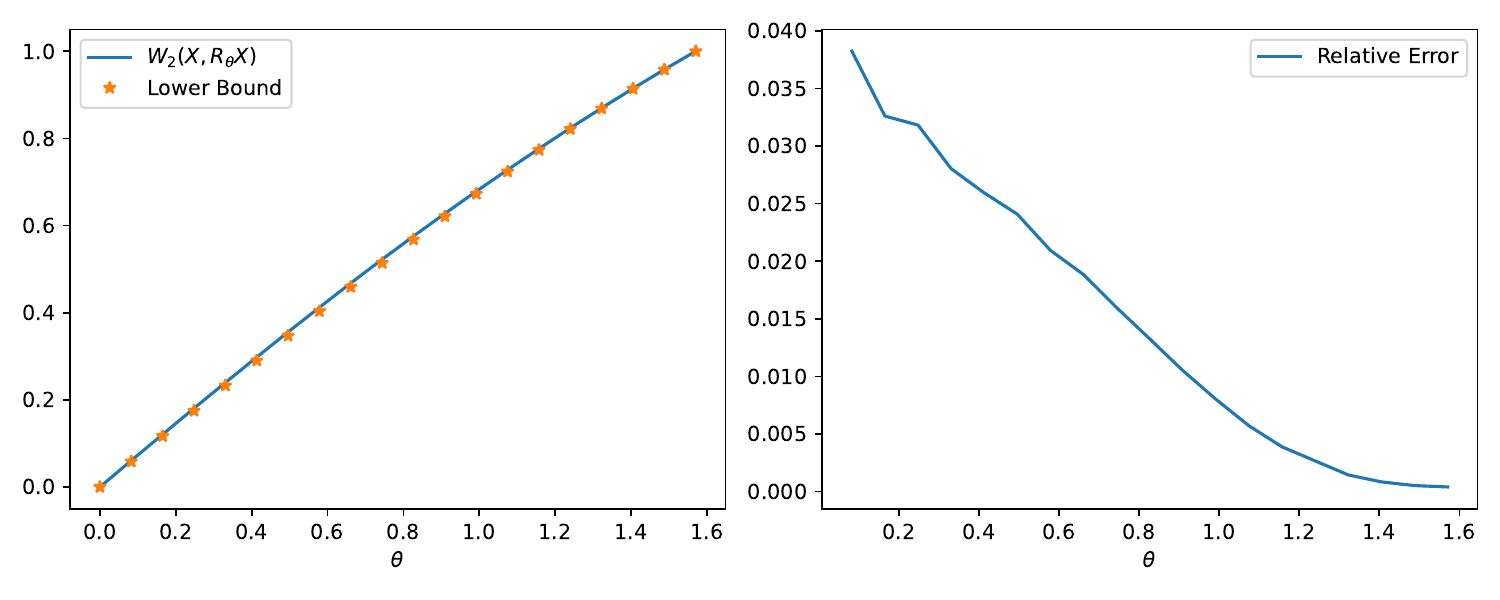}
\caption{Rotation of $[0,1]^2$ from Example \ref{EX:01Rotation}. (Left) $W_2(X,R_\theta X)$ and the lower bound $\sqrt{1-\cos(\theta)}$ for $\theta\in[0,\frac\pi2]$. (Right) Relative error $(W_2(X,R_\theta X)-\sqrt{1-\cos(\theta)})/W_2(X,R_\theta X)$ for $\theta\in[0,\frac{\pi}{2}]$.}\label{FIG:01Rotation}
\end{figure}

In the right side of Figure \ref{FIG:01Rotation}, we have taken out the first point $\theta=0$ because the relative error is large due to numerical artifacts. Similar numerical issues can be seen, e.g., in \cite{hamm2023manifold}.

\begin{example}\label{EX:CenteredCubeRotation}
Let $X$ be uniformly distributed over the centered unit square $[-\frac12,\frac12]^2\subset\R^2$. Then $\E[X_i]=0$ and $\Var[X_i]=\frac{1}{12}$ for $i=1,2$. Then the lower bound of Theorem \ref{THM:RotationW2} is $0$, and we only obtain the trivial bound $W_2(R_\theta X,R_\varphi X)^2\geq0$. Figure \ref{FIG:CenteredRotation} shows the Wasserstein distance of $X$ and $R_\theta X$ for $\theta\in[0,\frac{\pi}{2}]$. Note that due to the fact that the centered square has rotational symmetry at multiples of $\frac{\pi}{2}$, we see that the Wasserstein distance is increasing on $[0,\frac{\pi}{4}]$ and decreasing on $[\frac{\pi}{4},\frac{\pi}{2}]$ as expected.
 \begin{figure}[!h]
\centering
\includegraphics[width=.6\textwidth]{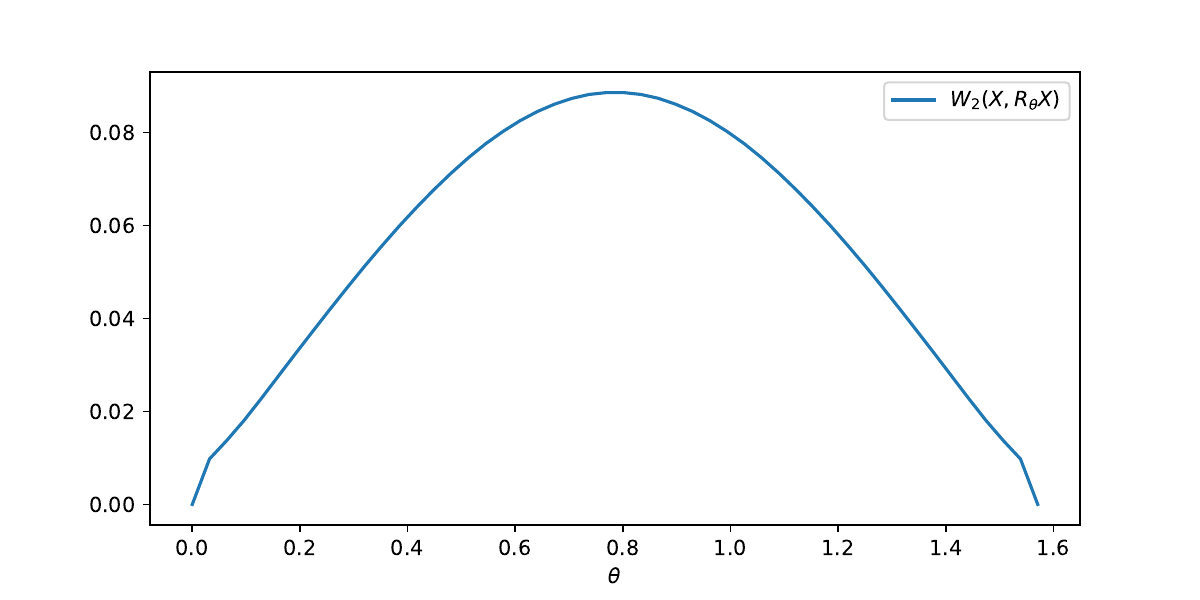}
\caption{Rotation of $[-\frac12,\frac12]^2$ from Example \ref{EX:CenteredCubeRotation}. Plotted is $W_2(X,R_\theta X)$ for $\theta\in[0,\frac{\pi}{2}]$.}\label{FIG:CenteredRotation}
\end{figure}
\end{example}

\begin{example}\label{EX:RectangleRotation} Let  $X$ be uniformly distributed on the rectangle  $[0,2]\times[0,1]\subset\R^2$ with densities 
$\frac{1}{2}\one_{[0,2]}({x_1}), \one_{[0,1]}(x_2)$ of its independent components $X_1,X_2$. Then 
$\E[X_1] = 1, \E[X_2] = \frac{1}{2}, \Var[X_1] = \frac{1}{3}, \Var[X_2] = \frac{1}{12}$ and the lower bound of Theorem \ref{THM:RotationW2} is 
\begin{equation}\label{EQN:RectangleLowerBound}W_2(R_\theta X,R_\varphi X)^2 \geq \frac{5}{2}(1 - \cos (\theta  - \varphi )) + \left(\frac{5}{6} - \sqrt{\frac{\cos^2(\theta  - \varphi )}{4} + \frac49} \right).\end{equation}

To illustrate the quality of the lower bound in this instance, Figure \ref{FIG:RectangleRotation} shows both the actual Wasserstein distance of $X$ and $R_\theta X$, the lower bound, and the difference with the lower bound for $\theta\in[0,\frac{\pi}{2}]$. We know that the rectangle does not give equality in the lower bound of \ref{THM:RotationW2}, but nonetheless it appears to be a relatively close approximation in this instance.

\begin{figure}[!h]
\centering
\includegraphics[width=.9\textwidth]{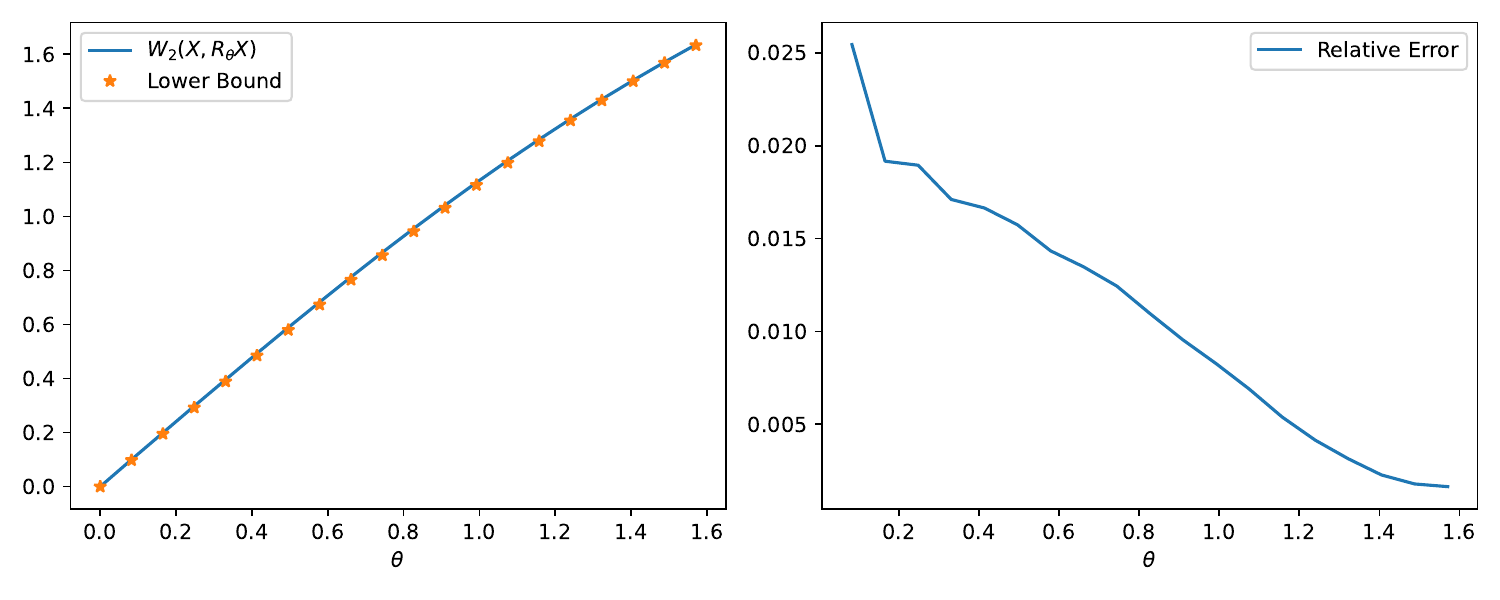}
\caption{Rotation of $[0,2]\times[0,1]$ from Example \ref{EX:RectangleRotation}. (Left) $W_2(X,R_\theta X)$ and the lower bound of \eqref{EQN:RectangleLowerBound} for $\theta\in[0,\frac\pi2]$. (Right) Relative error $(W_2(X,R_\theta X)-$lower bound)$/W_2(X,R_\theta X)$ for $\theta\in[0,\frac\pi2]$.}\label{FIG:RectangleRotation}
\end{figure}
\end{example}

We end this section with a proposition that gives an explicit representation of the map $T$ that gives the lower bound of Theorem \ref{THM:RotationW2} in the special case of random vectors with independent, zero-mean components (i.e., $\E[X_1]=\E[X_2]=\E[X_2X_2]=0$) in terms of the parameters discussed above. %Th Dowson and Landau's result (Theorem \ref{THM:FrechetBounds}, \cite{dowson1982frechet}).

\begin{proposition}
Let $X = \begin{bmatrix}X_1\\X_2\end{bmatrix}$ where $X_1$ and $X_2$ are uncorrelated random variables ($\E[X_1X_2]=0)$ with mean 0 and $\Var[X_1]=\E[X_1^2]=a$, $\Var[X_2]=\E[X_2^2]=b$ for $a,b>0$. Let $T := \Sigma_X^{-1}(\Sigma_X\Sigma_{R_\theta X})^\frac12$. Then 
\[T = \frac{1}{\sqrt{4ab+(a-b)^2\cos^2\theta}}\begin{bmatrix}b+a\cos^2\theta + b\sin^2\theta & (a-b)\cos\theta\sin\theta \\
(a-b)\cos\theta\sin\theta & a+b\cos^2\theta+a\sin^2\theta\end{bmatrix}\]
and $\det(T)=1.$ Moreover, 
\[\E|TX-X|^2 = 2(a+b)-2\sqrt{4ab+(a-b)^2\cos^2\theta}\leq W_2(X,R_\theta X)^2.\]
In particular, the map $T$ attains the lower bound for the Wasserstein distance between $X$ and $R_\theta X$ in \eqref{EQN:W2MeanTraceNorm}.
\end{proposition}
\begin{proof}
First, the fact that $W_2(X,R_\theta X)^2 \geq 2(a+b)-2\sqrt{4ab+(a-b)^2\cos^2\theta}$ follows directly from Theorem \ref{THM:RotationW2} with $\varphi=0$. The map $T = \Sigma_X^{-1}(\Sigma_X\Sigma_{R_\theta X})^\frac12$ is such that $\E|TX-X|^2 = \tr[\Sigma_X+\Sigma_{R_\theta X}-2(\Sigma_X\Sigma_{R_\theta X})^\frac12]$ by Theorem \ref{THM:FrechetBounds} (although one can verify this directly as well). However, it is not generally true that $TX = R_\theta X$ or that $T = R_\theta$ (consider, e.g., the case $\theta=\frac{\pi}{2}$, where $T = \diag(\sqrt{b/a},\sqrt{a/b})$ but $R_\theta = \begin{bmatrix}0&-1\\1&0\end{bmatrix}$).  

To compute $T$, recall that 
\[\Sigma_X = \diag(a,b),\qquad \Sigma_{R_\theta X} = \begin{bmatrix}a\cos^2\theta+b\sin^2\theta & (a-b)\cos\theta\sin\theta \\ (a-b)\cos\theta\sin\theta & b\cos^2\theta+a\sin^2\theta\end{bmatrix}.\]
Thus, $\tr(\Sigma_X\Sigma_{R_\theta X}) = (a^2+b^2)\cos^2\theta+2ab\sin^2\theta$ and $\det(\Sigma_X\Sigma_{R_\theta X}) = a^2b^2$, hence by Lemma \ref{LEM:SquareRoot}, \[\tr(\Sigma_X\Sigma_{R_\theta X})+2\sqrt{\det(\Sigma_X\Sigma_{R_\theta X})} = (a^2+b^2)\cos^2\theta+2ab\sin^2\theta+2ab = 4ab+(a-b)^2\cos^2\theta.\]
Therefore,
\[(\Sigma_X\Sigma_{R_\theta X})^\frac12 = \frac{1}{\sqrt{4ab+(a-b)^2\cos^2\theta}}\begin{bmatrix} a(a\cos^2\theta+b\sin^2\theta+b) & a(a-b)\cos\theta\sin\theta \\ b(a-b)\cos\theta\sin\theta & b(b\cos^2\theta+a\sin^2\theta+a)\end{bmatrix}.\]
Multiplying this matrix by $\Sigma_X^{-1} = \diag(\frac1a,\frac1b)$ gives the desired formula for $T$.
\end{proof}

The utility of the above proposition is that, even though the Wasserstein distance between $X$ and $R_\theta X$ is difficult to ascertain in general, we can easily form an explicit map $T$ that achieves the lower bound for the Wasserstein distance. Note that a general formula for $T$ could be given in the case of correlated components; in particular, the constant term outside of the matrix is $((a-b)^2+4c^2)\cos^2\theta+4(ab-c^2))^{-\frac12}.$ The matrix product $\Sigma_X\Sigma_{R_\theta X}$ however has a more complicated form given that $\Sigma_X$ is no longer diagonal. For ease of space, we do not include this result here.

However, we note that for an extreme case of functionally dependent components ($X_1=kX_2$) in which case $ab=c^2$ and equality in the Cauchy--Schwartz inequality is achieved, the square root term becomes $((a+b)^2\cos^2\theta)^{-\frac12}$ which is strictly greater than $((a-b)^2+4c^2)\cos^2\theta+4(ab-c^2))^{-\frac12}$ for random vectors which do not satisfy $ab=c^2$. This means that for functionally dependent random vectors, the constant term in $T$ is larger, and thus the Bures metric is larger.

\subsection{Compositions of affine maps}\label{SEC:Composition}

In this section we derive an upper bound for the Wasserstein distance between probability measures corresponding to a random vector and a composition of affine transformations. Recall that $T_\alpha$, $S_\lambda$, and $R_\theta$ represent translation, scaling, and rotation maps, respectively. We consider compositions of these, which can be utilized to understand a larger parametrized subset of diffeomorphisms.

\begin{theorem}\label{THM:CompositionW2}
Let $X = \begin{bmatrix}X_1 \\X_2\end{bmatrix}$ be a random vector in $\R^2$. Let $\alpha\in\R$, $\lambda\in\R^2$ with $\lambda_1,\lambda_2>0$, and $\theta\in[0,2\pi)$. Suppose also that equality holds in Theorem \ref{THM:RotationW2} for $X$ replaced with $S_\lambda X$ and $\varphi=0$ (which can occur, e.g., if $S_\lambda X$ is a Gaussian or elliptically contoured distribution). Then,
\begin{multline*}
    W_2(T_\alpha\circ R_\theta\circ S_\lambda X,X) \leq |\alpha| + \sqrt{(\lambda_1-1)^2\E[X_1^2]+(\lambda_2-1)^2\E[X_2^2]} \\ + \left\{2(\lambda_1^2\E[X_1]^2+\lambda_2^2\E[X_2]^2)(1-\cos\theta)+2(\lambda_1^2\Var[X_1]+\lambda_2^2\Var[X_2]) \right. \\
\left.  - 2\left[((\lambda_1^2\Var[X_1] - \lambda_2^2\Var[X_2])^2+4\lambda_1^2\lambda_2^2(\E[X_1X_2])^2) \cos^2(\theta)  \right. \right. \\ \left. \left. + 4\lambda_1^2\lambda_2^2(\Var[X_1]\Var[X_2]-(\E[X_1X_2])^2)\right]^\frac12
\right\}^\frac12
\end{multline*}
\end{theorem}

\begin{proof}
Note that \begin{align*}W_2(T_\alpha R_\theta S_\lambda X,X) & \leq W_2(S_\lambda X,X) + W_2(R_\theta S_\lambda X,S_\lambda X) + W_2(T_\alpha R_\theta S_\lambda X,R_\theta S_\lambda X) \\ & =: I_1+I_2+I_3.\end{align*}

By Proposition \ref{PROP:TranslationW2}, $I_3 = |\alpha|$, and $I_1 = \sqrt{(\lambda_1-1)^2\E[X_1^2]+(\lambda_2-1)^2\E[X_2^2]}$ by Proposition \ref{PROP:DilationW2} and Remark \ref{REM:Dilation}.
 
Finally, to estimate $I_2$, we apply Theorem \ref{THM:RotationW2} with $\varphi=0$ and $X$ replaced with $S_\lambda X$. Noting that $S_\lambda X = \begin{bmatrix} \lambda_1X_1 \\ \lambda_2X_2\end{bmatrix}$, $\E[(S_\lambda X)_i] = \lambda_i\E[X_i]$, $i=1,2$, and $\E[(S_\lambda X)_1(S_\lambda X)_2] = \lambda_1\lambda_2\E[X_1X_2]$, we have
\begin{multline*}I_2^2 = 2\left(\lambda_1^2(\E[X_1])^2 +
\lambda_2^2(\E[X_2])^2 \right)(1 - \cos\theta) +  2(\lambda_1^2\Var[X_1] + \lambda_2^2\Var[X_2]) \\ - 2\left\{((\lambda_1^2\Var[X_1] - \lambda_2^2\Var[X_2])^2+4\lambda_1^2\lambda_2^2(\E[X_1X_2])^2) \cos^2(\theta) \right. \\ + \left.4\lambda_1^2\lambda_2^2(\Var[X_1]\Var[X_2]-(\E[X_1X_2])^2)\right\}^\frac12.\end{multline*}
Combining the estimates above yields the result.
\end{proof}

\begin{theorem}\label{THM:CompositionW2UpperBound}
With the setup of Theorem \ref{THM:CompositionW2}, suppose that equality does not necessarily hold in Theorem \ref{THM:RotationW2}
\begin{multline*}
    W_2(T_\alpha\circ R_\theta\circ S_\lambda X,X) \leq |\alpha| + \sqrt{(\lambda_1-1)^2\E[X_1^2]+(\lambda_2-1)^2\E[X_2^2]} \\ + \left\{2(\lambda_1^2\E[X_1]^2+\lambda_2^2\E[X_2]^2)(1-\cos\theta)+2(\lambda_1^2\Var[X_1]+\lambda_2^2\Var[X_2]) \right. \\
\left.  + 2\left[((\lambda_1^2\Var[X_1] - \lambda_2^2\Var[X_2])^2+4\lambda_1^2\lambda_2^2(\E[X_1X_2])^2) \cos^2(\theta)  \right. \right. \\ \left. \left. + 4\lambda_1^2\lambda_2^2(\Var[X_1]\Var[X_2]-(\E[X_1X_2])^2)\right]^\frac12
\right\}^\frac12
\end{multline*}
\end{theorem}

The difference in the bounds of Theorems \ref{THM:CompositionW2} and \ref{THM:CompositionW2UpperBound} is merely the $+$ sign in the final term by using the upper bound of Dowson and Landau \cite{dowson1982frechet} that $\E|X-Y|^2\leq \tr[\Sigma_X+\Sigma_Y+2(\Sigma_X\Sigma_Y)^\frac12]$. This upper bound is typically far from optimal as can be seen in the examples above given that the lower bound is such a close approximation. Nonetheless it does give a qualitative upper bound for composition maps which is small for sufficiently small angles and dilations.

\begin{remark}
Note that one can interchange the order of any of the three operations (translation, scaling, rotation) in the theorems above and derive slightly different bounds. In applications such as handwriting analysis as discussed in the next section, putting translation last seems to make the most sense. However, the order of scaling and rotation could differ. For space consideration, we leave the version of Theorem \ref{THM:CompositionW2UpperBound} with scaling and rotation reversed to the interested reader, but mention that the difference is that the term $I_1$, now in terms of rotation, becomes merely that of Theorem \ref{THM:RotationW2}, while the middle term $I_2$ contains the variances $\E(R_\theta X)_1^2] = \E[X_1^2]\cos^2\theta+\E[X_2^2]\sin^2\theta-2\E[X_1X_2]\cos\theta\sin\theta$ and the similar term for the second component. 
\end{remark}

\begin{example}\label{EX:GaussianComposition}
Consider a bivariate zero-mean Gaussian random vector $X\sim \mathcal{N}(0,\Sigma)$ in $\R^2$. To illustrate the tightness of the upper bound of Theorem \ref{THM:CompositionW2} and how it may depend on correlation of the components of $X$, we generate 3 sets of 1,000 samples from three different zero-mean Gaussian distributions having covariance matrices
\[\Sigma_1 = \begin{bmatrix}8 & 4 \\ 4 & 4\end{bmatrix},\quad \Sigma_2 = \begin{bmatrix}8 & 2 \\ 2 & 4\end{bmatrix}, \quad \Sigma_3 =
 \begin{bmatrix}8 & 0 \\ 0 & 4\end{bmatrix}.\]
That is, we draw from Gaussians with highly correlated, moderately correlated, and uncorrelated components.

We set the translation vector $\alpha=0$, and the scaling vector $\lambda = \begin{bmatrix} \frac12 \\ 1\end{bmatrix}.$ Figure \ref{FIG:GaussianComposition} shows both the action of $T_\alpha\circ S_\lambda\circ R_\theta X$ on $X\sim\mathcal{N}(0,\Sigma_1)$ and the relative error (upper bound $- W_2(T_\alpha\circ R_\theta\circ S_\lambda X,X))/W_2(T_\alpha\circ R_\theta\circ S_\lambda X,X)$ for $X\sim(0,\Sigma_i)$, $i=1,2,3$ using the upper bound of Theorem \ref{THM:CompositionW2}. Error bars show one standard deviation of the error over 20 repeated trials. 

\begin{figure}[h!]
\centering
\includegraphics[width=\textwidth]{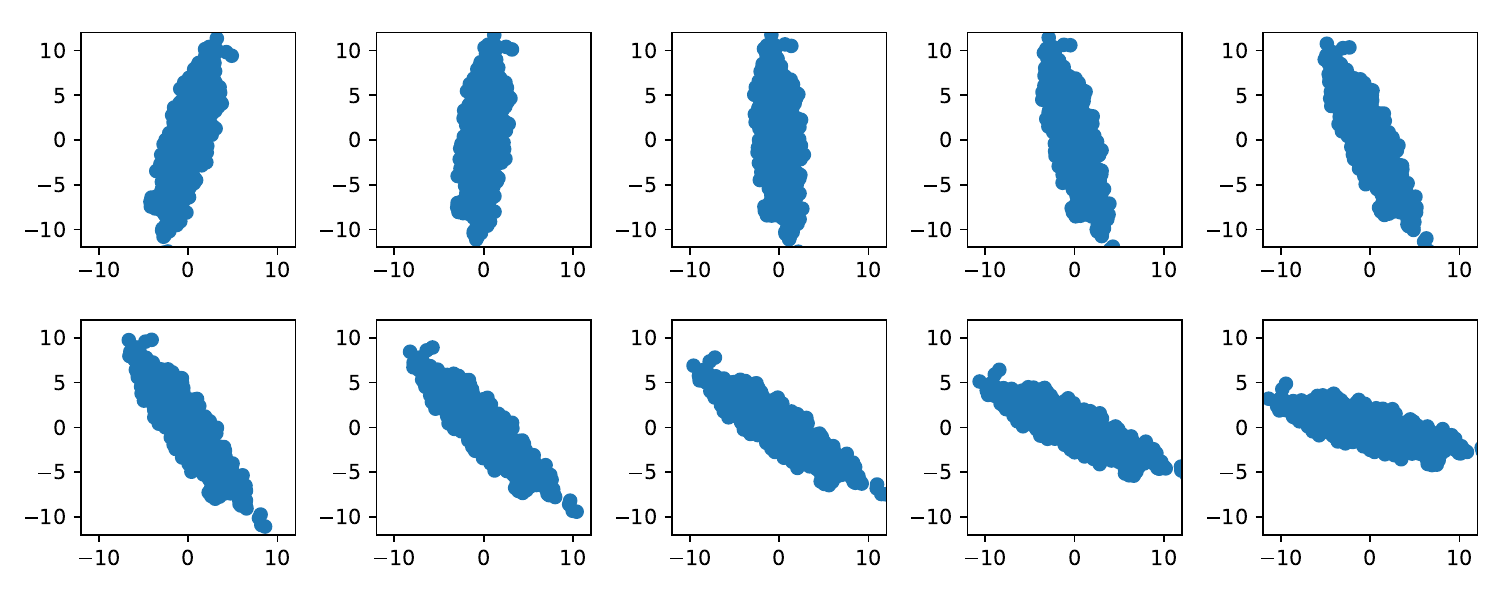}
\includegraphics[width=\textwidth]{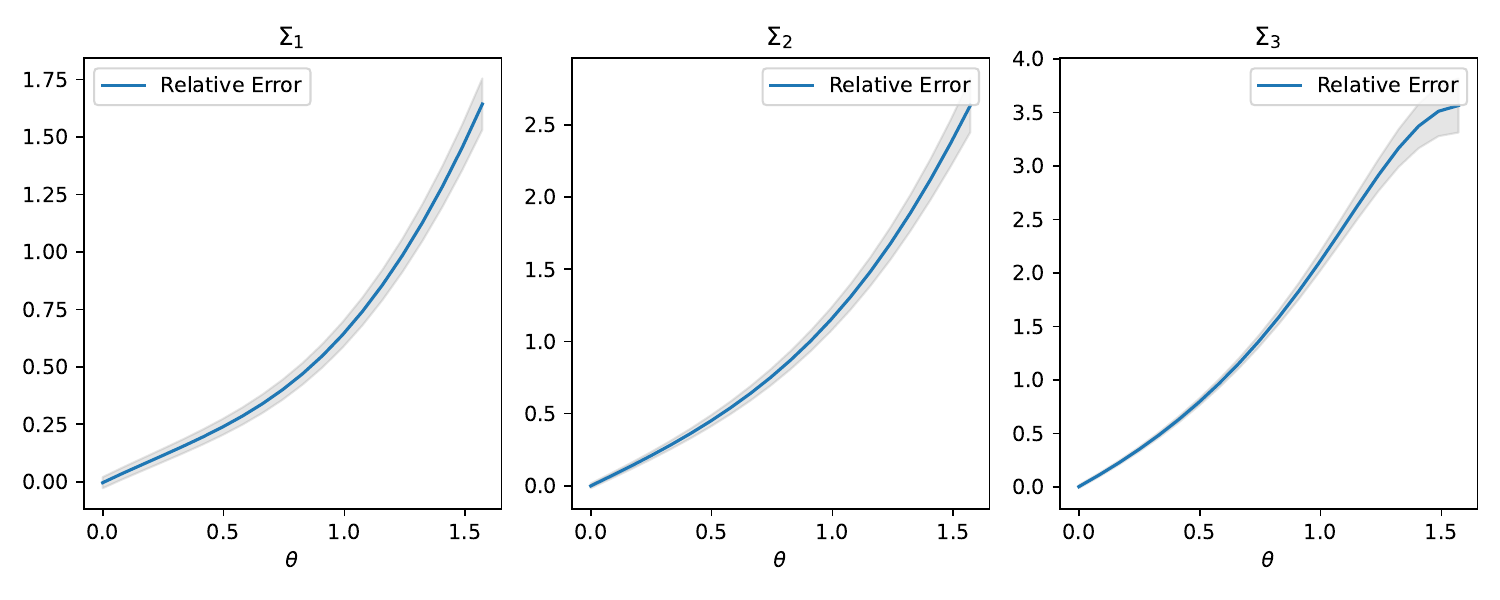}
\caption{Figure illustrating Example \ref{EX:GaussianComposition}. (Top) Sample of $X\sim\mathcal{N}(0,\Sigma_1)$ under the map $T_\alpha\circ R_\theta\circ S_\lambda$ for $\theta\in[0,\frac{\pi}{2}]$ with $\alpha = 0$, $\lambda = \begin{bmatrix} \frac12 \\ 2 \end{bmatrix}$. (Bottom) Relative error (upper bound $- W_2(T_\alpha\circ R_\theta\circ S_\lambda X,X))/W_2(T_\alpha\circ R_\theta\circ S_\lambda X,X)$ for $X\sim(0,\Sigma_i)$, $i=1,2,3$ using the upper bound of Theorem \ref{THM:CompositionW2} (which is valid because Gaussians satisfy equality in Theorem \ref{THM:RotationW2}).}\label{FIG:GaussianComposition}
\end{figure}

Intriguingly, the upper bound leads to smaller relative error for highly correlated Gaussians and larger error for uncorrelated Gaussians. However, with more experimentation, we find that this is not a general phenomenon. Indeed, when the scaling parameter is changed to $\lambda = \begin{bmatrix}2 \\ \frac12 \end{bmatrix}$ with the same covariance matrices, the relative error is slightly smaller for the uncorrelated Gaussian than the correlated ones. In all, this suggests that there may be a delicate interplay between the data distribution, the composition parameters, and the relative error of the upper bound which could be worthy of further experimentation.
\end{example}

\section{Applications to 1-dimensional submanifolds}

Now we will present a number of illustrative examples of Wasserstein distance under composition 
of scaling/rotation/translation for uniform distributions on a unit circle, line segment, and other curves. Such examples naturally arise in the context of alphabet letters recognition subject to affine distortions, especially when considering small perturbations from the actual objects.

\begin{example}\label{EX:CircleComposition}
    Consider a uniform measure on a unit circle, which can be realized as follows $(\Omega,F,P) = ([0,2\pi],\mathcal{B}([0,2\pi]),\frac{dt}{2\pi})$, where $\frac{dt}{2\pi}$ is the Lebesgue measure normalized to be a probability measure on $[0,2\pi]$, i.e., the uniform distribution with density $\frac{1}{2\pi}\one_{[0,2\pi]}(t).$  Consider a random vector \[X = \begin{bmatrix}X_1\\X_2\end{bmatrix}:\Omega\to S^1,\quad \textnormal{via}\quad X(t) = \begin{bmatrix} X_1(t)\\ X_2(t)\end{bmatrix} = \begin{bmatrix}\cos t\\\sin t\end{bmatrix}, t\in[0,2\pi].\]
    Then $X(t)$ is uniformly distributed over the unit circle $S^1$.

    Note that since $\int_0^{2\pi}\cos tdt = \int_0^{2\pi}\sin tdt = 0$, we have $\E[X]= 0.$ Similarly, $\E[X_1X_2]=0$, as $\int_0^{2\pi}\cos t\sin tdt = 0$, hence $X_1$ and $X_2$ are uncorrelated. Finally, 
    \[\Var[X_1] = \E[X_1^2] = \frac{1}{2\pi}\int_0^{2\pi}\cos^2tdt = \frac12 = \frac{1}{2\pi}\int_0^{2\pi}\sin^2tdt = \E[X_2^2] = \Var[X_2].\]
     Note that in this instance, $X$ satisfies the criteria for equality of the lower bound of Corollary \ref{COR:W2TraceEquality} (see Gelbrich \cite{gelbrich1990formula}). 
    Consequently, the conclusion of Theorem \ref{THM:CompositionW2} yields
    \begin{equation}\label{EQN:CircleCompositionBound}W_2(T_\alpha\circ R_\theta\circ S_\lambda X,X)\leq |\alpha|+\frac{1}{\sqrt{2}}|\lambda-\mathbf{1}| + \left(|\lambda|^2-\sqrt{(\lambda_1^2-\lambda_2^2)^2\cos^2\theta+4\lambda_1^2\lambda_2^2}\right)^\frac12.\end{equation}

    Note that this bound is more useful for small values of $\alpha,\theta,\lambda$, but nonetheless gives an explicit upper bound in terms of these parameters. Figure \ref{FIG:CircleLine} illustrates this example for 50 samples of the unit circle and 50 angles $\theta\in[0,2\pi)$ with $\alpha = \begin{bmatrix}-\frac12 \\ 1\end{bmatrix}$, and $\lambda = \begin{bmatrix} 2 \\ \frac12 \end{bmatrix}$. In this case, we see the symmetric behavior in the bound.

    % Figure \ref{FIG:CircleCompositionAlpha0} illustrates a simple case of this bound for 50 samples from the uniform distribution on $S^1$ and 50 values of $\theta\in[0,2\pi]$ for $\alpha = 0$ and $\lambda = \begin{bmatrix} 2 \\ \frac12 \end{bmatrix}$. In this case, we see equality at the endpoints as expected because there is no contribution from $|\alpha|$.  Figure \ref{FIG:CircleComposition} uses the same setup but with $\alpha = \begin{bmatrix}-\frac12 \\ 1\end{bmatrix}$. Here, we see that the translation yields a worse upper bound as expected, although the characteristic of the bound is qualitatively similar.

\begin{figure}[h!]
\centering
\includegraphics[width=.45\textwidth]{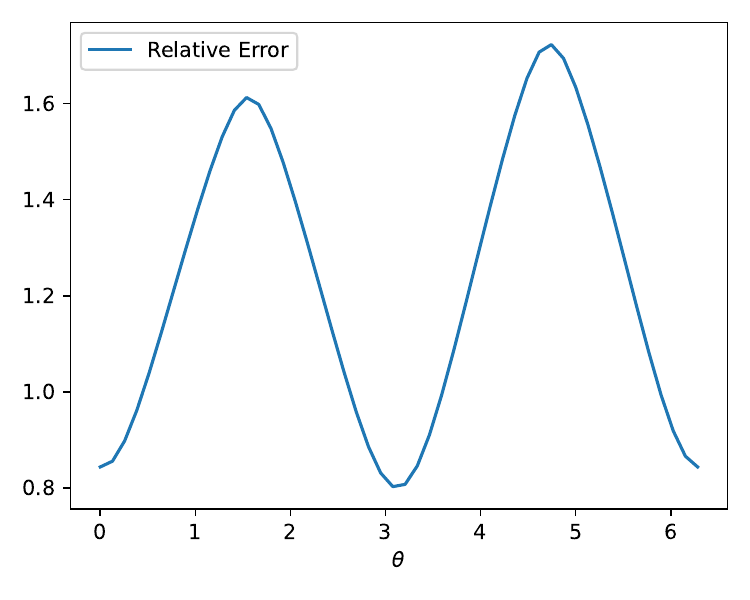}
\includegraphics[width=.45\textwidth]{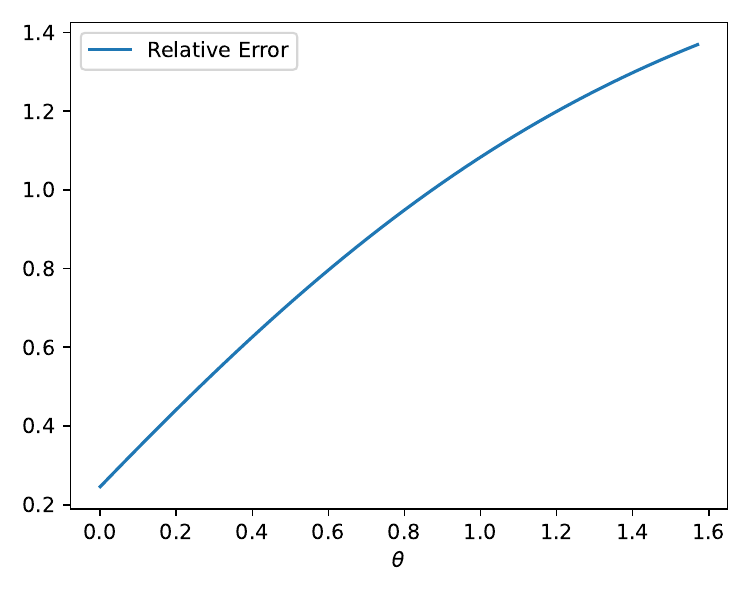}
\caption{(Left) $X$ drawn from the uniform distribution on the unit circle in $\R^2$ as in Example \ref{EX:CircleComposition}. Shown is the relative error between the Wasserstein distance and the upper bound. (Right) Relative error between Wasserstein distance and upper bound for $X$ drawn from the uniform distribution on $[-\frac12,\frac12]$ in $\R^2$ as in Example \ref{EX:LineComposition}. }\label{FIG:CircleLine}
\end{figure}

% \begin{figure}[h!]
% \centering
% \includegraphics[width=\textwidth]{CircleComposition_alpha0.pdf}
% \caption{$X$ drawn from the uniform distribution on the unit circle in $\R^2$ as in Example \ref{EX:CircleComposition}. (Left) $W_2(T_\alpha\circ R_\theta\circ S_\lambda X,X)$ and the upper bound of \eqref{EQN:CircleCompositionBound} for $\theta\in[0,2\pi]$ with $\alpha = 0$, $\lambda = \begin{bmatrix} 2 \\ \frac12 \end{bmatrix}$. (Right) Difference of the upper bound and $W_2(T_\alpha\circ R_\theta\circ S_\lambda X,X)$.}\label{FIG:CircleCompositionAlpha0}
% \end{figure}

% \begin{figure}[h!]
% \centering
% \includegraphics[width=\textwidth]{CircleComposition.pdf}
% \caption{$X$ drawn from the uniform distribution on the unit circle in $\R^2$ as in Example \ref{EX:CircleComposition}. (Left) $W_2(T_\alpha\circ R_\theta\circ S_\lambda X,X)$ and the upper bound of \eqref{EQN:CircleCompositionBound} for $\theta\in[0,2\pi]$ with $\alpha = \begin{bmatrix} -\frac12 \\ 1\end{bmatrix}$ and $\lambda = \begin{bmatrix} 2 \\ \frac12 \end{bmatrix}$. (Right) Difference of the upper bound and $W_2(T_\alpha\circ R_\theta\circ S_\lambda X,X)$.}\label{FIG:CircleComposition}
% \end{figure}
\end{example}

\begin{example}\label{EX:LineComposition}
    Let $X = \begin{bmatrix}X_1\\X_2\end{bmatrix} = \begin{bmatrix}X_1\\0\end{bmatrix}$, where $X$ is uniformly distributed on $[-\frac12,\frac12]$ with density $\one_{[-\frac12,\frac12]}(t).$ Then $\E[X_1]=\E[X_2]=0$, $\Var[X_1] = \frac{1}{12}$, $\Var[X_2] = 0$.  One can readily verify in this case that $\E[|X-R_\theta X|^2 = \E[X_1^2(2(1-\cos\theta))]=\frac16(1-\cos\theta)$. Additionally, one can compute the trace quantity in the lower bound of \eqref{EQN:W2MeanTraceNorm} directly and see that it is $\frac16(1-\cos\theta)$. Therefore, we may use Theorem \ref{THM:CompositionW2}, which implies that
    % \[W_2(T_\alpha\circ R_\theta\circ S_\lambda X,X)\leq |\alpha|+\frac{1}{\sqrt{12}}|\lambda| +\frac{1}{\sqrt{3}}|\lambda_1|\sin\frac{\theta}{2}.\]
    \begin{equation}\label{EQN:LineCompositionUpperBound}W_2(T_\alpha\circ R_\theta\circ S_\lambda X,X)\leq |\alpha|+\frac{1}{\sqrt{12}}|\lambda_1-1| + \frac{1}{\sqrt{3}}|\lambda_1|^2\sin\frac{\theta}{2}.\end{equation}
    Note that in deriving the bound above from Theorem \ref{THM:CompositionW2}. we have used the fact that $\sqrt{1-\cos\theta} = \sqrt{2}\sin\frac{\theta}{2}$.

Since $X$ is drawn from a 1-dimensional subset of $\R^2$, it makes sense that only dilation in the $x$-direction will impact the upper bound \eqref{EQN:LineCompositionUpperBound} (since dilation is the first operation). The upper bound here is evidently good for small translations and dilations as in Example \ref{EX:CircleComposition}, but will not be for large translation or dilation. Figure \ref{FIG:CircleLine} shows a simple example of this upper bound for $\alpha = \begin{bmatrix}-.25\\.1\end{bmatrix}$ and $\lambda = \begin{bmatrix} 2 \\ \frac12 \end{bmatrix}$, in which 50 samples are drawn from $\one_{[-\frac12,\frac12]}$ and 50 values of $\theta\in[0,\frac{\pi}{2}]$ are used.

% \begin{figure}[h!]
% \centering
% \includegraphics[width=\textwidth]{LineComposition.pdf}
% \caption{$X$ drawn from the uniform distribution on $[-\frac12,\frac12]$ in $\R^2$ as in Example \ref{EX:LineComposition}. (Left) $W_2(T_\alpha\circ R_\theta\circ S_\lambda X,X)$ and the upper bound of \eqref{EQN:LineCompositionUpperBound} for $\theta\in[0,2\pi]$ with $\alpha = \begin{bmatrix} -.25 \\ .1\end{bmatrix}$ and $\lambda = \begin{bmatrix} 2 \\ \frac12 \end{bmatrix}$. (Right) Difference of the upper bound and $W_2(T_\alpha\circ R_\theta\circ S_\lambda X,X)$.}\label{FIG:LineComposition}
% \end{figure}

\end{example}

\subsection{Toward handwriting analysis}

In the following examples we illustrate construction of uniform distributions over shapes represented by random curves with uncorrelated components, whereby satisfying the assumptions for Wasserstein distance formulas developed in this paper. 

We focus on deriving random vectors $X\in\R^2$ which could represent handwritten letters with the notion that these could be used in conjunction with the bounds in the previous sections to better understand Wasserstein distances between images in handwriting datasets like MNIST \cite{lecun1998mnist}. This has potential utility in a variety of avenues including clustering (e.g., \cite{Chen:2017aa}) or manifold learning. In particular, several recent works have explored the use of optimal transport in manifold learning (i.e., learning submanifolds of $W_2(\R^d)$) \cite{cloninger2023linearized,hamm2022wassmap,liu2022wasserstein,mathews2019molecular,negrini2023applications,wang2010optimal}. 

Utilizing the constructions below as well as template deformations such as the composition maps in Section \ref{SEC:Composition}, one could develop a synthetic handwritten image dataset for testing in machine learning applications. A preliminary example of this was done for handwritten $0$'s in \cite{hamm2022wassmap}.

Note that the simplest case for utilizing the theory above is when letters are generated using random vectors in $\R^2$ with uncorrelated components. As noted above, this simplifies some of the bounds considerably. However, this is not a requirement as seen in Theorem \ref{THM:RotationW2}, and therefore one has significant flexibility in producing curves to represent letters.

% \begin{example}[Square]\label{EX:Square}
% Let $(\Omega,F,P) = ([0,4],\mathcal{B}([0,4]),\frac{dt}{4})$, and consider
% \[X = \begin{bmatrix} X_1\\X_2\end{bmatrix},\quad \textnormal{where}\quad X(t) = \begin{bmatrix}X_1(t)\\X_2(t)\end{bmatrix},\quad 0\leq t\leq 4\]
% with
% \[X_1(t) = \begin{cases}
%     \frac12-t, & 0\leq t<1\\
%     -\frac12, & 1\leq t<2\\
%     t-\frac52,& 2\leq t<3\\
%     \frac12,& 3\leq t\leq4,
% \end{cases}\qquad 
% X_2(t) = \begin{cases}
%     -\frac12,& 0\leq t<1\\
%     t-\frac32,& 1\leq t<2\\
%     \frac12,& 2\leq t<3\\
%     \frac72-t, & 3\leq t\leq4.
% \end{cases}\]

% It can be readily verified that $\E[X_1] = \E[X_2] = \E[X_1X_2] = 0$.

% \begin{figure}[h!]
% \centering
% \includegraphics[width=.8\textwidth]{Square}
% \caption{(Left) $X_1$ and $X_2$ vs. $t$. (Right) $X = \begin{bmatrix}X_1\\X_2\end{bmatrix}$ forming a square.}
% \end{figure}
% \end{example}

\begin{example}[Letter C]\label{EX:LetterC}

Let $(\Omega,F,P) = ([\frac\pi2,\frac{3\pi}{2},\mathcal{B}([\frac\pi2,\frac{3\pi}{2}]),\frac{dt}{\pi})$, and let
\[X(t) = \begin{bmatrix} X_1(t)\\X_2(t)\end{bmatrix} = \begin{bmatrix} \cos t+\frac2\pi\\ \sin t\end{bmatrix},\quad t\in\left[\frac\pi2,\frac{3\pi}{2}\right].\]
$X$ is the semicircle of radius 1 centered at $(\frac2\pi,0)$, and one can readily check that $\E[X_1]=\E[X_2]=\E[X_1X_2]=0$.

\begin{figure}[h!]
\centering
\includegraphics[width=.8\textwidth]{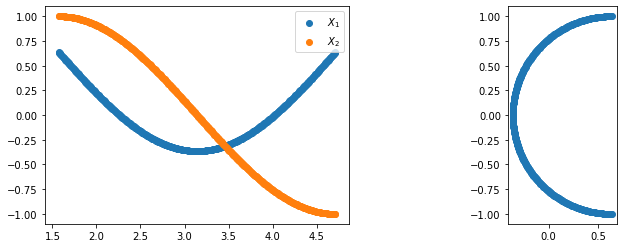}
\caption{(Left) $X_1$ and $X_2$ vs. $t$. (Right) $X = \begin{bmatrix}X_1\\X_2\end{bmatrix}$ forming the letter C.}
\end{figure}
\end{example}

\begin{example}[Letter A]\label{EX:LetterA}

Let $(\Omega, F, P) = ([0,6],\mathcal{B}([0,6]),\frac{dt}{6})$, and consider $X$ defined as $X(t) = \begin{bmatrix} X_1(t)\\ X_2(t)\end{bmatrix}$, $t\in[0,6]$, where
\[X_1(t) = \begin{cases}
    2-t,& 0\leq t\leq4\\
    t-5,& 4<t\leq 6,
\end{cases}\qquad X_2(t) = \begin{cases}
    t-1,& 0\leq t\leq2\\
    3-t,& 2<t\leq4\\
    0,& 4<t\leq6.
\end{cases}\]

Here $\E[X_1]=\E[X_2]=\E[X_1X_2]=0$, and $\E[X_1^2] = 1$, $\E[X_2^2]=\frac29.$
% Here $\E[X_1]=\E[X_2]=\E[X_1X_2]=0$, and $\E[X_1^2] = 6/6$, $\E[X_2^2]=\frac43 /6.$

\begin{figure}[h!]
\centering
\includegraphics[width=.8\textwidth]{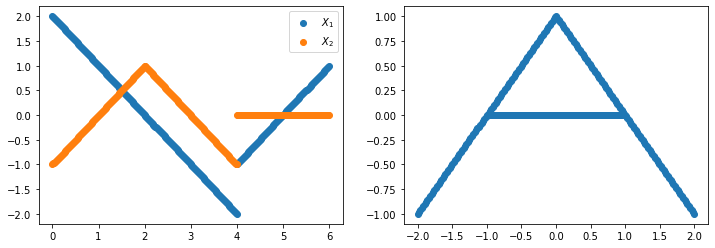}
\caption{(Left) $X_1$ and $X_2$ vs. $t$. (Right) $X = \begin{bmatrix}X_1\\X_2\end{bmatrix}$ forming the letter A.}\label{FIG:ARotation}
\end{figure}

For comparison with simpler distributions above, we illustrate rotations of the letter A here. From the above quantities, the lower bound according to Theorem \ref{THM:RotationW2} for a rotated A is
\[W_2(X,R_\theta X)^2 \geq \frac{22}{9}-2\sqrt{\frac{49}{81}\cos^2\theta+\frac89}.\]

Interestingly, we see from Figure \ref{FIG:ARotation} that the relative error for comparing the letter A with its rotated copy is similarly small when compared to those of Figures \ref{FIG:01Rotation} and \ref{FIG:RectangleRotation} for rotated cubes and rectangles.

\begin{figure}[h!]
\centering
\includegraphics[width=\textwidth]{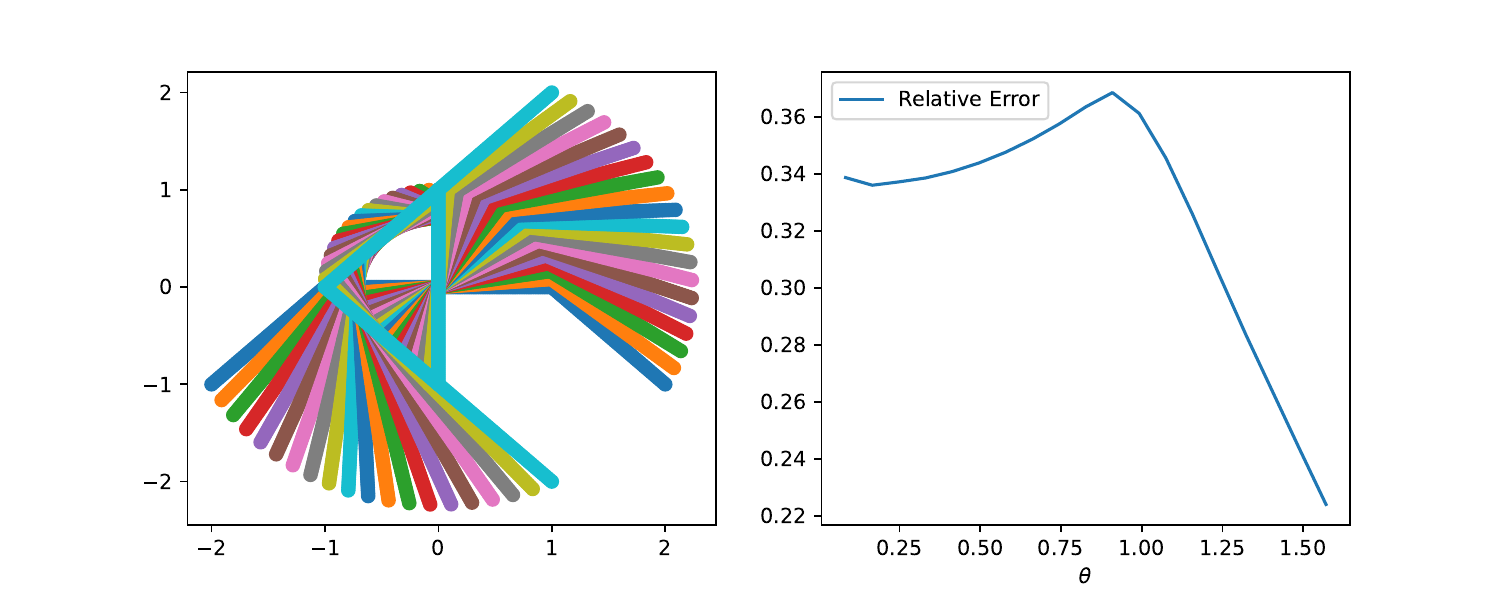}
\caption{(Left) Letter A undergoing rotation. (Right) Relative error $(W_2(R_\theta X,X)-$lower bound)$/W_2(R_\theta X,X)$ for the lower bound of Theorem \ref{THM:RotationW2}.}
\end{figure}

\end{example}

\begin{example}[Letter T]
The examples above are of letters whose components are mean 0 and uncorrelated. Here, we consider two ways of making the letter T, with and without these restrictions. We then consider the relative error for the corresponding lower bound of rotated T's.

Let $(\Omega,F,P) = ([0,4],\mathcal{B}([0,4]),\frac{dt}{4})$ with 
\begin{equation}\label{EQN:T1}X_1(t) = \begin{cases} 1-t,& 0\leq t\leq 2\\
0,& 2<t\leq 4,\end{cases},\qquad X_2(t) = \begin{cases}
    \frac12,& 0\leq t\leq2\\
    t-\frac72,& 2<t\leq4.
\end{cases}\end{equation}
    It is readily verified that $\E[X_1]=\E[X_2]=\E[X_1X_2]=0.$

A second parametrization is
\begin{equation}\label{EQN:T2}X_1(t) = \begin{cases} 1-t,& 0\leq t\leq 2\\
0,& 2<t\leq 4,\end{cases},\qquad X_2(t) = \begin{cases}
    1,& 0\leq t\leq2\\
    t-3,& 2<t\leq4.
\end{cases}\end{equation}
One can compute $\E[X_1] = 0$, $\E[X_2] = \frac12$, $\E[X_1^2] = \frac16$, $\E[X_2^2] = \frac43.$

Figure \ref{FIG:T} shows both drawings of the letter T using these parametrizations. To make the distribution mean 0, the first parametrization makes the T non-centered, i.e., it starts at a height of -1.5 and the top is at a height of .5. The second parametrization is designed so that the top of the T is at a height of 1 and the bottom is at a height of -1, and the width of the T is also 2.

\begin{figure}[h!]
\centering
\includegraphics[width=.8\textwidth]{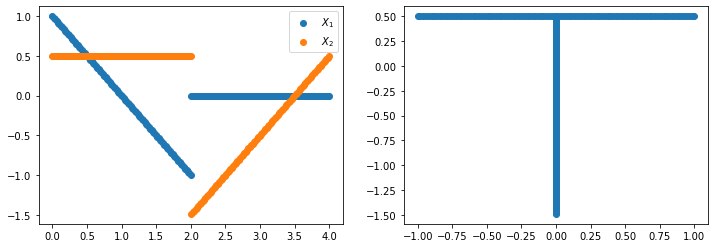}
\includegraphics[width=.97\textwidth]{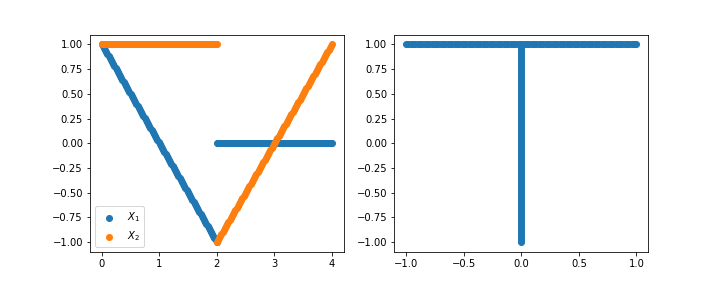}
\caption{(Left) $X_1$ and $X_2$ vs. $t$ for the parametrization \eqref{EQN:T1} (top) and \eqref{EQN:T2} (bottom). (Right) $X = \begin{bmatrix}X_1\\X_2\end{bmatrix}$ forming the letter T.}\label{FIG:T}
\end{figure}
\end{example}

Now, Figure \ref{FIG:TRotation} shows the letter T under rotation by $\theta\in[0,\frac{\pi}{2}]$ and the relative error with the lower bound of Theorem \ref{THM:RotationW2} for both parametrizations of the letter.

\begin{figure}[h!]
\centering
\includegraphics[width=.6\textwidth]{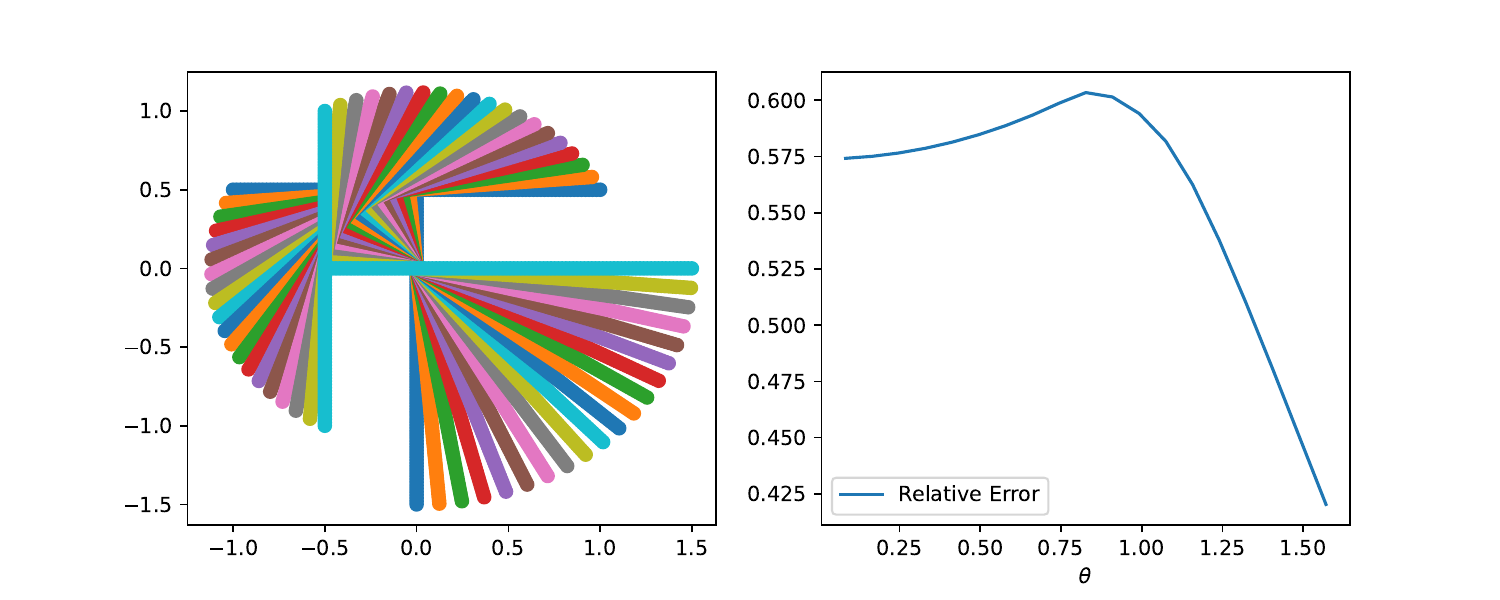} \hspace{-2em}\includegraphics[width=.3\textwidth]{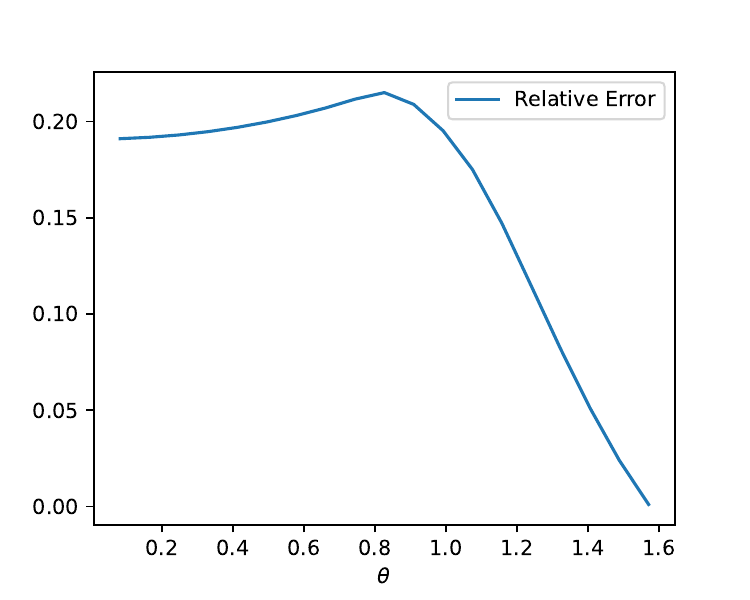}
\caption{(Left) Letter T undergoing rotation. (Center) Relative error $(W_2(R_\theta X,X)-$lower bound)$/W_2(R_\theta X,X)$ for the lower bound of Theorem \ref{THM:RotationW2} for the T parametrized by \eqref{EQN:T1} (Right) Relative error for the T with nonzero mean parametrized by \eqref{EQN:T2}.}\label{FIG:TRotation}
\end{figure}

One sees that more symmetry in the letter parametrized by \eqref{EQN:T2} yields a smaller relative error with the lower bound. In general, this may suggest that if one is to design synthetic experiments or datasets for handwritten digits, care must be taken in how it is done, as seemingly small choices such as what domain the curves are generated in, how the letters are centered in the image domain, and symmetry, among other factors, may have a significant impact in the Wasserstein geometry of sets of deformations of such letters designed to mimic differences of handwriting between people.

\subsection{Wassmap embeddings}

Some dimensionality reduction algorithms such as multidimensional scaling (MDS) \cite{mardia1979multivariate} and Isomap \cite{tenenbaum2000global} involve computing a square distance matrix to capture local geometry of data. In particular, MDS requires a matrix of the form $D_{ij} = d(x_i,x_j)^2$ where $d$ is some metric on the data space and $x_i$ are the given data. The truncated singular value decomposition of a centered version of this square distance matrix is used to embed the data into a lower-dimensional Euclidean space.

Classical MDS used Euclidean distance as the metric $d$, but several works have studied these embeddings when $d$ is the Wasserstein metric $W_2$ \cite{cloninger2023linearized,hamm2022wassmap,negrini2023applications,wang2010optimal}. MDS or Isomap with Wasserstein distance matrices is referred to as Wassmap in \cite{hamm2022wassmap}. We refer the interested reader to these works for detailed descriptions of these algorithms and their particulars.  For a small demonstration on the synthetic letter T, we consider the Wassmap embedding of rotations of T parametrized by both \eqref{EQN:T1} and \eqref{EQN:T2} above. We also show the MDS embedding in which $D_{ij}$ is the value of the lower bound of Theorem \ref{THM:RotationW2} for comparison. Results are shown in Figure \ref{FIG:Wassmap}.

\begin{figure}[h!]
\centering
\includegraphics[width=\textwidth]{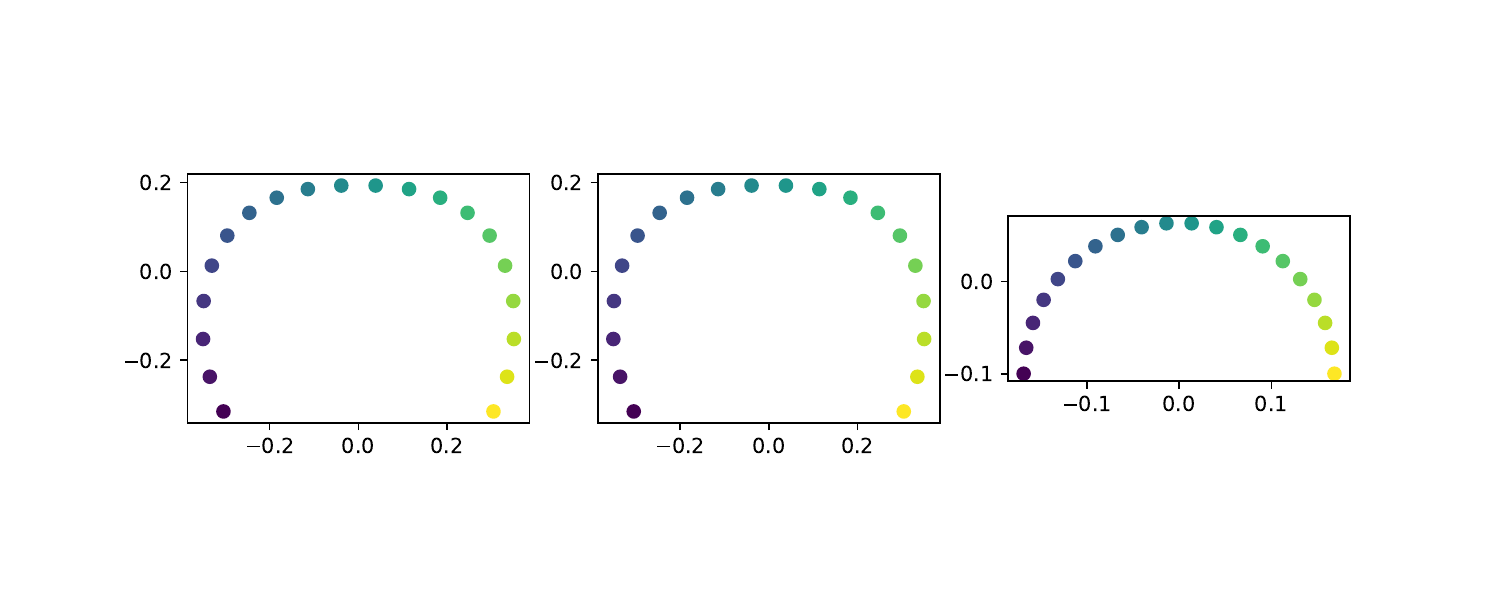}
\caption{(Left) Wassmap embedding of rotated copies of letter T parametrized by \eqref{EQN:T1}. (Center) Same for T parametrized by \eqref{EQN:T2}.  (Right) MDS embedding using the lower bound of Theorem \ref{THM:RotationW2} for T parametrized by \eqref{EQN:T1}.}\label{FIG:Wassmap}
\end{figure}

Experiments in \cite{hamm2022wassmap} show that in some cases, the rotational structure of a set of data generated by rotating a single initial image can be uncovered by the embedding which appears to be part of a circle. Here, we see that this structure does indeed appear, and in fact the embedding of both parametrizations of the letter T appear essentially indistinguishable. Further experimentation needs to be done in the future, but it would be interesting to better understand the nuances of how image geometry and Wasserstein distance between deformations of images impact dimensionality reduction embeddings.

\section{Conclusion}

We have explored some lower and upper bounds of Wasserstein distances between random vectors and translations, rotations, dilations, and compositions of these. Our analysis is based on the bounds of Dowson and Landau \cite{dowson1982frechet} and Gelbrich \cite{gelbrich1990formula}. In particular, our concrete bounds are in terms of mean, variance, and covariance of the given random vectors, which are easily computed or approximated.

We have illustrated the quality of these bounds for a variety of random vectors including correlated and uncorrelated Gaussians, 1-dimensional curves, and imitations of handwritten digits.  Finally, we briefly consider some downstream applications to dimensionality reduction of handwritten digits. Our final experiments raise a number of questions to be considered in the future, including
\begin{itemize}
\item What is the effect of centering and symmetry of handwritten letters on analyzing a synthetic handwritten letter dataset?
\item What is the effect of these things on dimensionality reduction embeddings of such a synthetic dataset?
\item Are there other factors of images that substantially affect the quality of the lower bounds considered here?
\end{itemize}

% We have shown in this subsection how several letters of the alphabet can be formed by random vectors in $\R^2$ that satisfy the criteria in previous sections which allow us to determine Wasserstein distance between these letters and certain deformations of them. In particular, Theorem \ref{THM:CompositionW2} applied to these examples allows one to understand the Wasserstein distance between an idealized letter A (for instance) and a dilated, rotated, and translated copy of it.  

% \begin{figure}[h!]
% \centering
% \includegraphics[width=\textwidth]{TRotation2.pdf}
% \caption{(Left) Letter T undergoing rotation. (Right) Relative error $(W_2(R_\theta X,X)-$lower bound)$/W_2(R_\theta X,X)$ for the lower bound of Theorem \ref{THM:RotationW2}.}
% \end{figure}

\section*{Acknowledgements}

KH was partially supported by a Research Enhancement Program grant from the College of Science at the University of Texas at Arlington. Research was sponsored by the Army Research Office and was accomplished under Grant
Number W911NF-23-1-0213. The views and conclusions contained in this document are those of the authors and
should not be interpreted as representing the official policies, either expressed or implied, of the Army Research
Office or the U.S. Government. The U.S. Government is authorized to reproduce and distribute reprints for
Government purposes notwithstanding any copyright notation herein.

The authors thank the reviewers for valuable suggestions that improved the results and presentation of this manuscript.

%%%%%%%%%%%%%%%%%%%%%%%%%%%%%%%%%%%%%%%%%%%%%%%%
%%%%%%%%%%%%%%%%%%%%%%%%%%%%%%%%%%%%%%%%%%%%%%%%
%%%%%%%%%%%%   Bibliography           %%%%%%%%%%
%%%%%%%%%%%%%%%%%%%%%%%%%%%%%%%%%%%%%%%%%%%%%%%%
%%%%%%%%%%%%%%%%%%%%%%%%%%%%%%%%%%%%%%%%%%%%%%%%
\bibliographystyle{plain}
%\bibliography{wass}
%%%%%%%%%%%%%%%%%%%%%%%%%%%%%%%%%%%%%%%%%%%%%%%%

% Experiments to run:
% - Composition with correlated Gaussian
% - Composition with correlated ellipse
% - Rotated A 0 mean
% - Rotated A nonzero mean

\end{document}